\documentclass{article}

\usepackage{microtype}
\usepackage{graphicx}
\usepackage{subcaption}
\usepackage{booktabs} 
\usepackage{wrapfig}
\usepackage{multirow}
\providecommand{\eg}{\textit{e.g.}\ }
\providecommand{\ie}{\textit{i.e.}\ }

\usepackage{algorithm}
\usepackage{algorithmic}

\usepackage{hyperref}

\usepackage[accepted]{icml2026}

\usepackage{amsmath}
\usepackage{amssymb}
\usepackage{mathtools}
\usepackage{amsthm}

\usepackage[capitalize,noabbrev]{cleveref}

\theoremstyle{plain}
\newtheorem{theorem}{Theorem}[section]

\theoremstyle{definition}

\theoremstyle{remark}

\usepackage[textsize=tiny]{todonotes}

\icmltitlerunning{Normality Calibration in Semi-supervised Graph Anomaly Detection}

\begin{document}

\twocolumn[
  \icmltitle{Normality Calibration in Semi-supervised Graph Anomaly Detection}




\icmlsetsymbol{equal}{*}
  \icmlsetsymbol{cor}{$\dagger$}

  \begin{icmlauthorlist}
    \icmlauthor{Guolei Zeng}{equal,ox,smu}
    \icmlauthor{Hezhe Qiao}{equal,smu}
    \icmlauthor{Guoguo Ai}{nj}
    \icmlauthor{Jinsong Guo}{cor,unli}
    \icmlauthor{Guansong Pang}{cor,smu}
  \end{icmlauthorlist}

  \icmlaffiliation{ox}{University of Oxford}
  \icmlaffiliation{smu}{Singapore Management University}
  \icmlaffiliation{nj}{Nanjing University of Science and Technology}
  \icmlaffiliation{unli}{Unlimidata Limited}

  \icmlcorrespondingauthor{Guansong Pang}{gspang@smu.edu.sg}
  \icmlcorrespondingauthor{Jinsong Guo}{jinsong.guo@unlimidata.com}


  \icmlkeywords{Machine Learning, ICML}

  \vskip 0.3in
]



\printAffiliationsAndNotice{\icmlEqualContribution}

\begin{abstract}
Semi-supervised graph anomaly detection (GAD), which assumes a subset of labeled normal nodes for training, is widely studied. However, existing methods learn normality only from these labeled nodes, often overfitting their patterns and causing high detection errors, such as many false positives. To overcome this limitation, we propose \textbf{GraphNC}, a \underline{graph} \underline{n}ormality \underline{c}alibration framework that leverages both labeled and unlabeled data to calibrate the normality from a teacher (a pre-trained semi-supervised GAD model) jointly in anomaly score and representation spaces.
GraphNC includes two main components, anomaly \underline{score} \underline{d}istribution \underline{a}lignment (\textbf{ScoreDA}) and perturbation-based \underline{norm}ality \underline{reg}ularization (\textbf{NormReg}). ScoreDA optimizes our model’s anomaly scores by aligning them with the teacher’s score distribution. Because the teacher provides accurate scores for most normal nodes and some anomalies, this alignment pulls the scores of the two classes toward opposite ends, making them more separable. To reduce the impact of inaccurate teacher scores, NormReg regularizes normality in the representation space, making normal node representations more compact via a perturbation-guided consistency loss applied only to the labeled nodes. Comprehensive experiments on six benchmarks demonstrate that GraphNC (1) consistently and substantially enhances the performance of teacher models from different GAD methods (2) achieves new state-of-the-art performance. Our code is available at
\renewcommand\UrlFont{\color{blue}} \url{https://github.com/mala-lab/GraphNC}.

\end{abstract}

\section{Introduction}

Graph anomaly detection (GAD), aiming to identify irregular nodes in a graph, has gained increasing attention due to its critical role in real applications such as the detection of spams in social networks and frauds in financial networks \citep{pang2021deep, ma2021comprehensive, qiao2025deep, yan2025address}. 
Semi-supervised GAD, which leverages a subset of labeled normal nodes to learn normal patterns, has gained significant attention in real-world applications \citep{qiao2024generative,ai2026semi}. This GAD setting is practical in that normal nodes often account for the majority in a graph, and thus, it is significantly less costly to obtain compared to abnormal samples that are scarce and their occurrence typically involves high financial/reputation loss.  

However, compared to unsupervised GAD that is more popular, less studies have been done on semi-supervised GAD. Furthermore, existing semi-supervised GAD methods are primarily built on the limited labeled normal nodes, preventing them from learning complete normal patterns, thus inclining to overfit the annotated normality \citep{ding2019deep,wang2021one, ding2021inductive, fan2020anomalydae, qiao2024generative,ai2026semi}.  As a result, these methods suffer from high detection errors—\eg normal nodes that are dissimilar to the labeled normal nodes are detected as anomalies (false positives) and vice versa (false negatives), especially the former one—due to the large overlapped anomaly scores for the normal  and abnormal nodes. These issues can be observed in the results of a recent state-of-the-art (SOTA) method GGAD \citep{qiao2024generative} in Figs. \ref{fig:example} (a)
and (b).
The same phenomenon can also be observed in other semi-supervised methods like data reconstruction and one-class classification-based methods (see App. \ref{app:Additional}). 

\begin{figure}[t]
  \centering
  \includegraphics[width=\columnwidth]{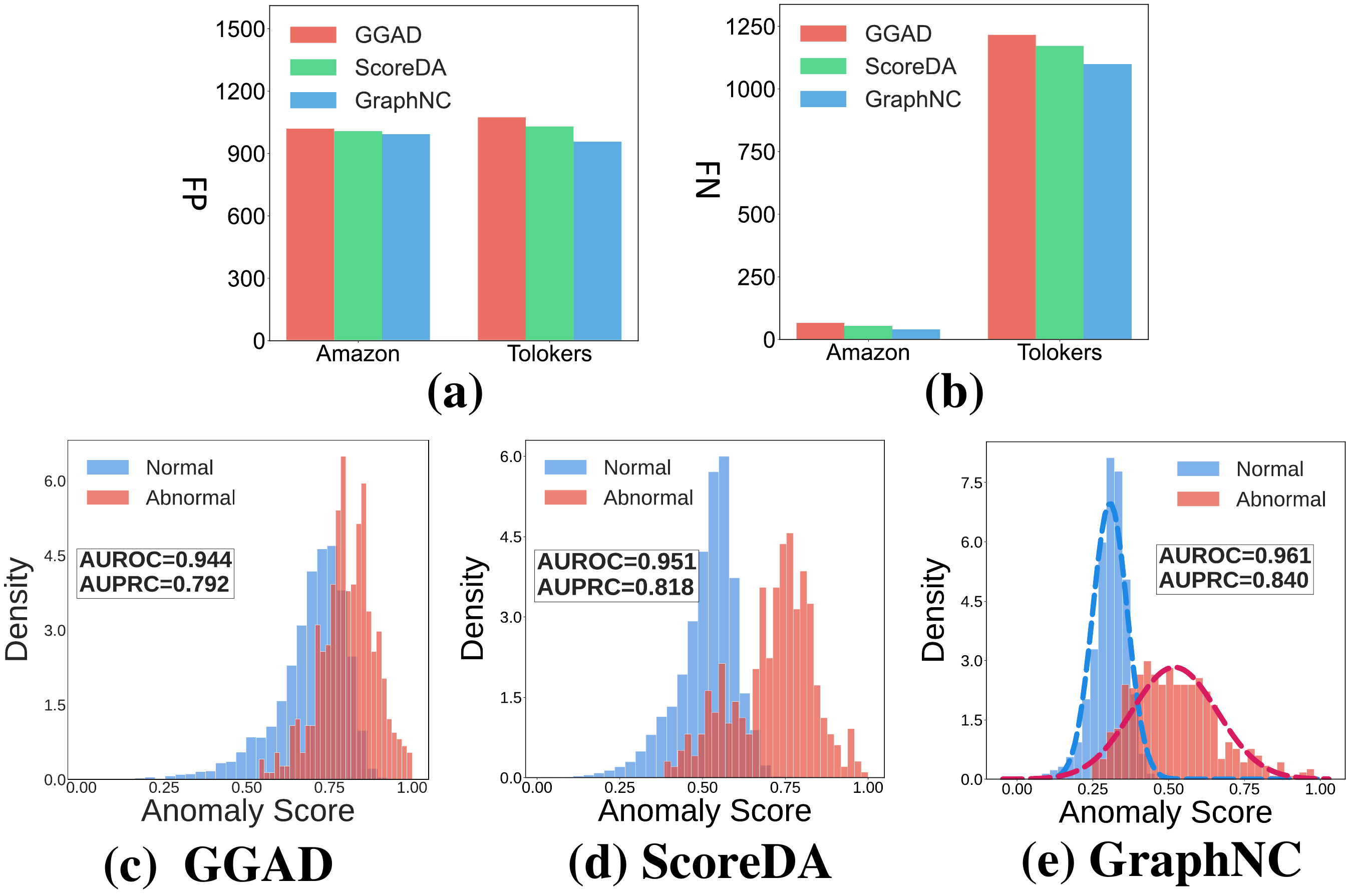}
  \caption{\textbf{(a)} False positive (FP) and \textbf{(b)} False negative (FN) results on Amazon \protect\citep{dou2020enhancing} and Tolokers \protect\citep{mcauley2015image}.   \textbf{(c)}, \textbf{(d)}, and \textbf{(e)} show the score distributions of normal and abnormal nodes for GGAD, ScoreDA, and ScoreDA+NormReg (\ie, GraphNC) on Amazon, where GGAD is used as a teacher model in both ScoreDA and GraphNC. }
  \label{fig:example}
\end{figure}

To address these issues, we propose \textbf{GraphNC}, a novel \underline{graph} \underline{n}ormality \underline{c}alibration framework that exploits both labeled and unlabeled data to calibrate the normality learned from a teacher model (\ie, a pre-trained semi-supervised GAD model) jointly in anomaly score and representation spaces. GraphNC includes two main components, namely anomaly \underline{score} \underline{d}istribution \underline{a}lignment (\textbf{ScoreDA}) and perturbation-based \underline{norm}ality \underline{reg}ularization (\textbf{NormReg}). ScoreDA optimizes the anomaly scores of our model by aligning them with the score distribution yielded by the teacher model. Due to accurate scores in most of the normal nodes and part of the anomaly nodes in the teacher model, the score alignment pulls the anomaly scores of the normal and abnormal classes toward the two ends, resulting in more separable anomaly scores, as shown in Fig. \ref{fig:example} (d). This helps reduce false positive and false negative rates.

However, due to their inherent limitation mentioned above, the teacher models inevitably produce some inaccurate anomaly scores, which can mislead the score optimization in ScoreDA. To mitigate this issue, NormReg is devised to regularize the graph normality in the representation space, enabling the student model to refine itself via a perturbation-guided consistency loss applied only to the labeled nodes. This normality calibration helps learn more compact normal representations, pulling the anomaly scores of normal nodes closer to each other and toward the lower end of the anomaly score distribution, as shown in Fig. \ref{fig:example} (e).

By jointly optimizing these two components, GraphNC learns more generalized normality representations and more discriminative anomaly scores. Besides, GraphNC is a generic framework where different pre-trained teacher models can be plugged into,  and it can obtain better GAD performance if the teacher model is stronger.
In summary, our contributions are as follows
 \begin{itemize}
\item We introduce GraphNC, a novel graph normality calibration framework that leverages both labeled normal nodes and unlabeled data to calibrate normality learning from the teacher model in both the score and feature spaces for semi-supervised GAD .

\item In GraphNC, we introduce two new components: anomaly score distribution alignment (ScoreDA) and perturbation-based normality regularization (NormReg). ScoreDA aligns the scores of our model with those of the pre-trained teacher model, which helps push the anomaly scores of the two classes toward opposite ends. NormReg mitigates the impact of inaccurate teacher scores by enforcing representation compactness among normal nodes during alignment.

\item GraphNC is a flexible framework where different teacher models can be plugged and played, 
achieving consistently enhanced GAD improvement across three types of teacher models, including data reconstruction, one-class-based, and anomaly-generation-based approaches. 
This is verified by our 
comprehensive experiments on six benchmark datasets.

\end{itemize}

\section{Related Work}

\noindent\textbf{Graph Anomaly Detection.} Existing GAD methods can be broadly categorized into unsupervised, semi-supervised, and fully supervised approaches \citep{ma2021comprehensive, qiao2025deep}. Unsupervised methods typically assume that no labeled nodes are known, and they learn normal patterns through some proxy task designs such as reconstruction, one-class classification, and adversarial learning \citep{ding2019deep,wang2021one,ding2021inductive}. Although this setting is widely used in anomaly detection, they overlook the abundance of normal samples, where normal nodes are less costly to obtain than the abnormal nodes. The fully supervised methods, on the other hand, rely on labeled normal and abnormal nodes and formulate the task as an imbalanced classification problem \citep{liu2021pick,tang2022rethinking,gao2023addressing,gao2023alleviating,chen2024consistency, liu2025survey}. These supervised methods rely on large labeled data,
limiting their applicability in real-world scenarios, especially those where anomalies are very infrequent and very costly to collect; they may overfit labeled anomalies, failing to generalize to unseen anomalies \citep{wang2023open}. Semi-supervised methods are generally more practical for real-world applications, as only a subset of labeled normal nodes is required \citep{zhou2022unseen,qiao2024generative,qiao2025deep, ai2026semi}.

Only limited work has been done in this line. GGAD \citep{qiao2024generative} and RHO \cite{ai2026semi} are two recent works. GGAD \citep{qiao2024generative} utilizes the labeled normal nodes to generate synthetic outliers that mimic real anomalies, and then trains a binary classifier. It also shows that popular unsupervised methods \cite{ding2019deep, wang2021one, ding2021inductive, hezhe2023truncated} can be adapted to the semi-supervised settings by minimizing their loss on the normal data only. 
RHO \citep{ai2026semi} aims to address the limitation that existing graph filters may fail to fit subsets of normal nodes with varying levels of homophily in semi-supervised GAD by learning adaptive spectral filters. However, both RHO and related methods mainly learn normality from the labeled normal nodes, which may make the model prone to overfitting the annotated normality patterns.

\noindent\textbf{Semi-supervised Learning.} Semi-supervised learning aim to leverage a small amount of labeled data together with a large amount of unlabeled data to train models, aiming to achieve high performance while reducing the cost of data annotation \citep{yang2022survey, chen2022semi}. Representative semi-supervised learning approaches include pseudo-labeling and consistency-based methods. The pseudo-labeling with confidence thresholding is highly successful and widely-adopted, which is typically built on the smoothness assumption that neighboring data points in the feature space tend to exhibit similar labels \citep{grandvalet2004semi, chen2023softmatch}. However, such a mechanism requires a quantity–quality trade-off, which undermines the learning process, since incorporating more unlabeled data helps increase data coverage but often introduces noisy or unreliable pseudo-labels \citep{pham2021meta,kage2024review}. 
The consistency-based model enforces prediction consistency under stochastic noise or data augmentation \citep{laine2016temporal}. Recent studies such as RankMatch \citep{mai2024rankmatch}, InterLUDE \citep{huang2024interlude}, and SCHOOL \citep{mo2024revisiting}, have been proposed in different data modalities, including visual images, text, and graphs, aiming to encourage stable predictions across views, thereby enhancing generalization and reducing reliance on labeled data \citep{gui2024survey}. However, they are mainly focused on classification tasks, different from our anomaly detection task that has only one-class labels and unbounded distribution in the anomaly class.



Many semi-supervised anomaly detection models that are trained on normal data have been proposed for anomaly detection tasks, but they are focused on visual data \citep{wu2024deep,cao2024survey}, failing to capture the complex structural information in the graph. There are also methods that aim to leverage small labeled anomaly data and large unlabeled data, such as DevNet \citep{pang2019deep}, Deep SAD \citep{ruff2019deep}, DPLAN \citep{pang2021toward}, PReNet \citep{pang2023deep}, and RoSAS \citep{xu2023rosas}, which explore a different problem setting from ours.

\section{Problem Statement}

\noindent \textbf{Notations.} Given an attributed graph $\mathcal{G} = (\mathcal{V}, \mathcal{E}, \mathbf{X})$, where $\mathcal{V}$ denotes the node set with $v_{i} \in \mathcal{V}$ and $\left | \mathcal{V}  \right | = N$ representing the total number of the node, $\mathcal{E}$ denotes the edge set, and $ \mathbf{X} =\left[\mathbf{x}_{1}, \mathbf{x}_{2}, \ldots, \mathbf{x}_{N}\right] \in \mathbb{R}^{N \times M}$ is a set of node attributes. Each node $v_{i}$ has a $M$-dimensional attribute $\mathbf{x}_{i} \in \mathbb{R}^M$.  The topological structure of $\mathcal{G}$ is represented by an adjacency matrix $\mathbf{A} \in \mathbb{R}^{N \times N}$. 

\textbf{Semi-supervised GAD.} Let $\mathcal{V}_a$, $\mathcal{V}_n$
be two disjoint subsets of $\mathcal{V}$, where $\mathcal{V}_a$ represents abnormal node set and $\mathcal{V}_n$ represents normal node set, and typically the number of normal nodes is significantly greater than the abnormal nodes, \ie, $|\mathcal{V}_n| \gg |\mathcal{V}_a|$, then the goal of semi-supervised GAD is to learn the mapping function $f \to \mathbb{R}$, such that $f(v) < f(v')$, where $\forall v \in {{\cal V}_n},{v^\prime } \in {{\cal V}_a}$ , given a set of labeled normal nodes $\mathcal{V}_l \subset \mathcal{V}_n$, with $|\mathcal{V}_l| = R$, and no access to any labels of the abnormal nodes. $\mathcal{V}_u = \mathcal{V}/ \mathcal{V}_l$ is the set of the unlabeled nodes and used as test data.

\textbf{Graph Neural Networks for Representation Learning.}
Graph Neural Networks (GNNs) are typically employed to learn the representation of each
node due to their strong representation ability. This can be formulated as follows
\begin{equation}
    \mathbf{H}^{(\ell)}=\mathrm{GNN}\left(\mathbf{A}, \mathbf{H}^{(\ell-1)} ; \mathbf{W}^{(\ell)}\right)
\end{equation}

where  $\mathbf{W}^{(\ell)}$ are the parameters, which are updated during training, and $\mathbf{H}^{(\ell)} = \{\textbf{h}_1,\textbf{h}_2, 
..., \textbf{h}_N\}$ denote the $l$-th (final) layer embeddings, which are the representation of all nodes at layer $l$. In this paper, our student model adopts a 2-layer GNN as the backbone for representation learning, while the teacher model uses the recommended GNN architecture in its original work.

\section{Methodology}

\begin{figure*}[t]
  \centering
  \includegraphics[width=\textwidth]{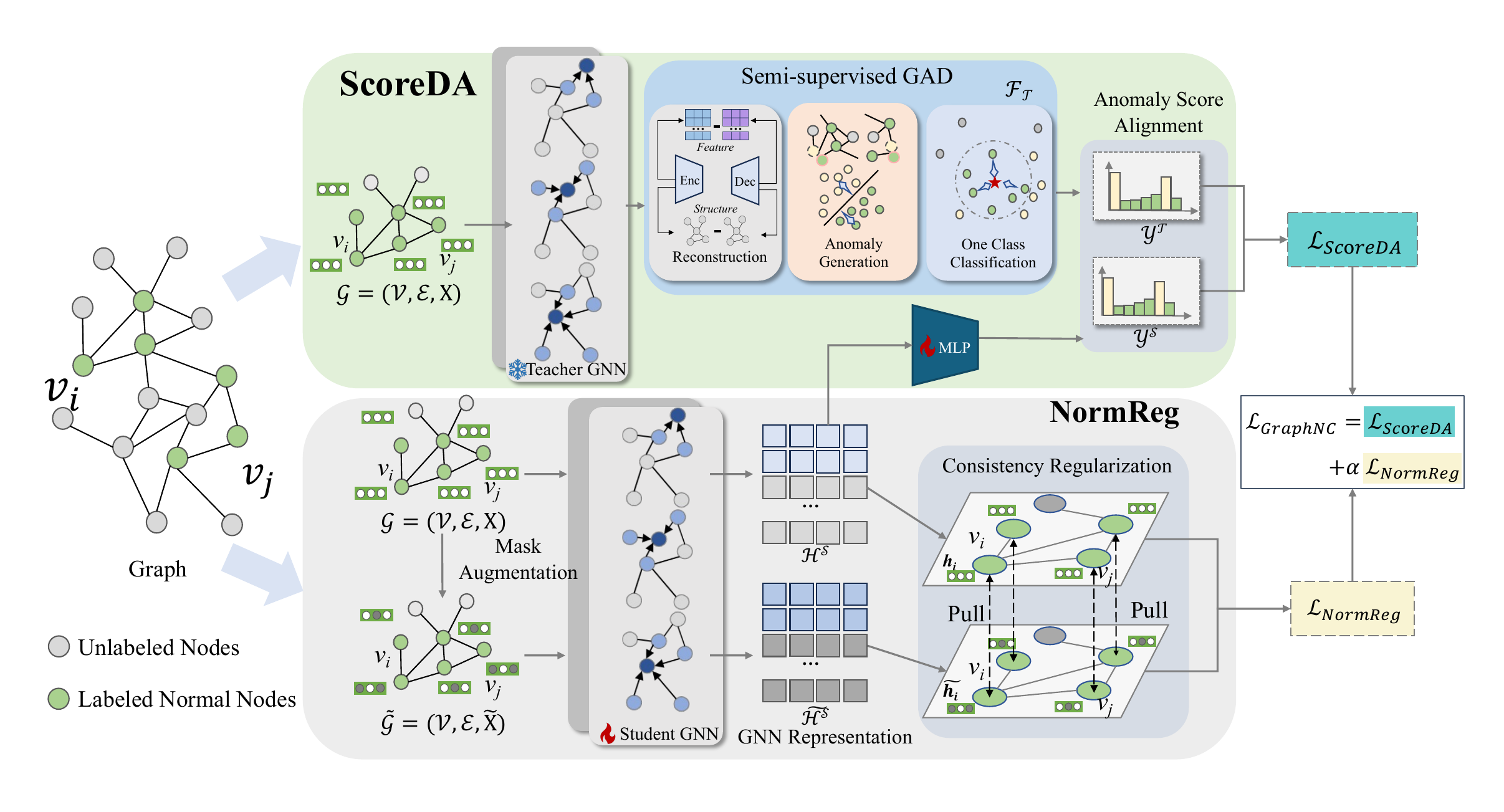}
  \caption{Overview of GraphNC. The input graph consists of a small labeled normal node set and a large unlabeled node set. GraphNC is based on a teacher-student network framework. ScoreDA aims to align the anomaly scores of the student network with the anomaly scores yielded by an existing semi-supervised GAD method to calibrate the normality in the score space. NormReg is introduced to utilize a consistency regularization loss function in representation space to learn more compact representations of normal nodes, thereby reducing the negative impact of inaccurate anomaly scores that may be produced by the teacher model. The teacher network is pre-trained first, and frozen afterward, with only the student network being trained when jointly optimizing ScoreDA and NormReg.}
  \label{fig:framework}
\end{figure*}


\subsection{Overview of the Proposed GraphNC}
The overview of the proposed GraphNC is shown in the Fig. \ref{fig:framework}. Given a pre-trained semi-supervised GAD model, it includes two steps during training. (1) Given a graph with only a set of labeled normal nodes, ScoreDA aligns the scores of our student model with the anomaly score distribution produced by the pre-trained detector as the teacher model, thereby achieving normality calibration in the score space.
(2) NormReg subsequently complements score alignment by mitigating noise in the teacher model’s anomaly scores through a consistency constraint in the representation space on the labeled normal nodes, encouraging them to learn more compact normality representations.
During inference, the output scores from the student model are used to derive the anomaly scores.

\subsection{Anomaly Score Distribution Alignment (ScoreDA)}
The simplistic normal pattern extraction in existing semi-supervised GAD models, trained on only a limited set of labeled normal samples, often is prone to overfit the patterns in the labeled data, resulting in high detection errors, such as elevated false positive rates, \ie, many normal samples are misclassify as anomalies. We leverage ScoreDA to align the output scores of our model with the score distribution produced by the teacher model. As the anomaly score distribution of the pre-trained teacher model can correctly characterize most nodes, including most normal nodes and part of the abnormal nodes in the score space, by aligning with this teacher scores, the student model is enforced to push scores for the normal and abnormal classes toward the two opposite ends, resulting in more separable anomaly scores. This effectively reduces the misclassification of normal samples, leading to a lower false positive rate, while also yielding a decrease in the false negative rate.

Formally, let $\mathcal{F}_\mathcal{T}$ be the pre-trained teacher model,
$\mathcal{Y}^\mathcal{T}= \mathcal{F}_{\mathcal{T}}{(\mathbf{X}, \mathcal{G}; \Theta)}$ be the output scores for the entire node set, where $\Theta$ are the trainable parameters, and $\mathcal{Y}^\mathcal{T} =\{y_1^\mathcal{T},y_2^\mathcal{T}, ...,y_N^\mathcal{T} \}$ be the anomaly score distribution of samples,
then we leverage a student model $\mathcal{F}_{S}$ with parameter $\Phi$ which is implemented as the combination of GNN and MLP layers (\ie, $\mathcal{F}_{GNN}$ with parameters $\Omega$ and $\mathcal{F}_{MLP}$ with parameters $\phi$ respectively) to align the score set $\mathcal{Y}^\mathcal{T}$. 
For the student model, let $\mathcal{H}^{S}$ be the output of the GNN student model, for a node $v_i$, the representation is then denoted by ${\bf{h}}_i^{\mathcal{S}} = \mathcal{F}_{GNN}({\bf{x}}_i, \mathcal{G}; \Omega)$. On top of that, we apply a MLP layer to obtain the predicted score, $y^\mathcal{S}_i = \mathcal{F}_{MLP}({\bf{h}}_i^{\mathcal{S}}; \phi)$.  Finally,  we employ the MSE loss to minimize the discrepancy between the anomaly score from the student model $\mathcal{Y}^\mathcal{S}$ and  $\mathcal{Y}^\mathcal{T}$ yielded by the teacher model as follows:
\begin{equation}
    \mathcal L_{ScoreDA} = \frac{1}{|{{\cal V}}|}\sum_{v_i\in {\mathcal{V}}} \|y^\mathcal{S}_i- y^\mathcal{T}_i\|_2^2,
\end{equation}
where $|{\cal V}|$ is the number of all nodes, including both labeled and unlabeled nodes. By minimizing the score discrepancy between the teacher and the student, the accurate scores  in $\mathcal{Y}^\mathcal{T}$ helps enforce the clusters of the anomaly scores in the normal and abnormal classes in the two opposite ends, resulting in more separable anomaly scores.

\subsection{Perturbation-Guided Normality Regularization (NormReg)}


Due to the inherent limitations in normality characterization and the reliance on a limited set of normal samples, teacher models inevitably produce some inaccurate scores. Solely applying score alignment can enforce the fitting of the student model to these inaccurate scores, thereby misleading the optimization of the student model.
To address this challenge, we further introduce NormReg to 
calibrate the learned normality in the representation space based on solely the labeled normal nodes, enabling the student to not only learn from the teacher but also self-refine.

To this end,
we utilize an node attribute-based masking mechanism that randomly masks a proportion $\omega$ of the 
attributes on labeled normal nodes to learn consistent representations for these nodes. This masking creates an augmented graph $\tilde{\mathcal{G}}$, which contains many variations of the labeled
\begin{figure}[t]
  \centering
  \includegraphics[width=\columnwidth]{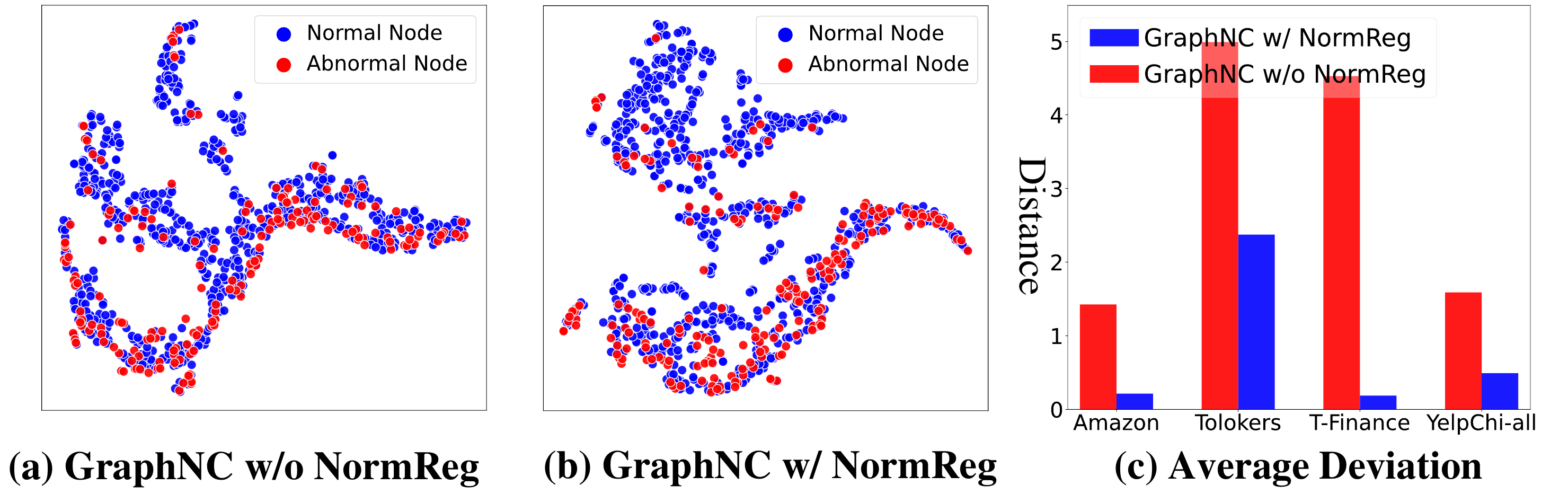}
  \caption{(a) and (b) provide t-SNE visualization of the node representations for GraphNC with/without using NormReg. (c) The average deviation of the normal class on Tolokers.}
  \label{fig:homo}
\end{figure}
normal nodes, simulating diverse normal patterns that may be different from the ones derived directly from the labeled nodes. Thus, enforcing consistency over the representations of these augmented nodes and the labeled normal nodes helps cluster the normality from both the labeled nodes and their augmented versions closer. Formally, let $\mathcal{H}_i^{\mathcal{S}}$ and $\tilde{\mathcal{H}}_i^{\mathcal{S}}$ be the original and augmented node representations in the student model respectively,
NormReg then minimizes the discrepancy between the representations of the original node and its augmented counterpart by optimizing the following loss function $\mathcal{L}_{NormReg}$:
\begin{equation}
{{\cal L}_{{NormReg}}} =  \frac{1}{{\left| {{{\cal V}_l}} \right|}}\sum\limits_{{v_i} \in {\cal V}_l} {\left\| {{{\bf{h}}_i^{\mathcal{S}}} - {{\widetilde {\bf{h}}}_i^{\mathcal{S}}}} \right\|_2^2},
\end{equation}
where  $\mathcal{V}_l$ are the labeled normal nodes, and ${{\bf{h}}_i^{\mathcal{S}}} \in \mathcal{H}_i^{\mathcal{S}}$and ${{\widetilde {\bf{h}}}_i^{\mathcal{S}}} \in \tilde{\mathcal{H}}_i^{\mathcal{S}}$ are the original and augmented representations for node $v_i$ respectively. 
As shown in Figs. \ref{fig:homo} (a) and (b), by adding ${\cal L}_{{NormReg}}$, the representation distributions of the normal nodes obtained by GraphNC
are more compact compared to GraphNC without using ${\cal L}_{{NormReg}}$.
This demonstrates that the regularization effectively mitigates the negative impact of inaccurate anomaly scores from the teacher model. To provide more evidence, we also compute average deviation of normal nodes to the normal class prototype on four datasets, and it is clear that the deviation is significantly reduced after applying NormReg, 
as illustrated in Fig. \ref{fig:homo} (c).
Beyond the empirical demonstrations, our theoretical analysis also show that NormReg can well complement ScoreDA in lowering the detection errors.


\begin{theorem}\label{theorem}
Let $\sigma_\mathcal{T}^2$ and $\sigma_\mathcal{S}^2$ denote the variance of the anomaly scores yielded by the teacher model and the student model in the normal class, respectively, then under our GraphNC framework, the student model achieves shrinking score variance (\ie, $\sigma_\mathcal{S}^2 < \sigma_\mathcal{T}^2$) by minimizing $\mathcal{L}_{NormReg}$ while applying $\mathcal{L}_{ScoreDA}$, thereby reducing False Positive Rate (FPR) and False Negative Rate (FNR).
\end{theorem}

The proof can be found in App. \ref{theorem_analyse}. According to Theorem \ref{theorem}, minimizing $\mathcal{L}_{NormReg}$ complements $\mathcal{L}_{ScoreDA}$ in reducing the intra-class score variance of normal classes in the student model, which can also be observed when comparing the scores of the normal class in Fig. \ref{fig:example}(e) to that in Fig. \ref{fig:example}(d),
lowering both FPR and FNR, especially the FPR.

\subsection{Training and Inference}
\noindent \textbf{Training.} During training, the student model is optimized by aligning its score distribution with that of the teacher model, while also minimizing the consistency regularization.  To be specific, the total loss
function $\mathcal{L}_{GraphNC}$ is formulated as a combination of the $\mathcal{L}_{ScoreDA}$ and  $\mathcal{L}_{NormReg}$:
\begin{equation}
    \mathcal{L}_{GraphNC} =  \mathcal{L}_{ScoreDA} + \alpha \mathcal{L}_{NormReg},
\end{equation}
where $\alpha$ is a hyper-parameter that adjusts the influence of normality consistency regularization in optimization.

\noindent \textbf{Inference.} During inference, since the student model $\mathcal{F}_{\mathcal{S}}$ calibrates the normality of the teacher and yields more accurate scores, its output scores can be directly employed as an anomaly score, which is defined as:
\begin{equation}
    S(v_i) = \mathcal{F}_{ \mathcal{S}
    }(v_i, {\bf{x}}_i, \mathcal{G}, \Phi^*),
\end{equation}
where $\Phi^*=\{\Omega^*, \phi^* \}$ is the optimized parameter set of the student model.

\begin{table*}[t]
\caption{AUROC and AUPRC results on six real-world GAD datasets. For each dataset, the best performance per column within each metric is boldfaced, with the second-best underlined. ``Avg'' denotes the average performance of each method.}
\label{gen_auroc}
\begin{center}
\resizebox{0.95\textwidth}{!}{
\begin{tabular}{c|c|cccccc|c}
\hline
\hline
\multirow{2}*{\textbf{Metric}}&\multirow{2}*{\textbf{Method}} & \multicolumn{6}{c|}{\textbf{Dataset}}\\
&&Amazon &T-Finance &Reddit &YelpChi &Tolokers &Photo &Avg.\\
\hline
\multirow{10}{*}{AUROC}
&DOMINANT &0.8867&0.6167&0.5194&\underline{0.6314}&0.5121&0.5314&0.6162\\
&AnomalyDAE    &0.9171&0.6027&0.5280&0.6161&0.6047&0.5272&0.6326\\
&OCGNN   &0.8810&0.5742&0.5622&0.6155&0.4803&0.6461&0.6265\\
&AEGIS     &0.7593&0.6728&0.5605&0.6022&0.4451&0.5936&0.6055\\
&GAAN    &0.6531&0.3636&0.5349&0.5605&0.3785&0.4355&0.4876\\
&TAM    &0.8405&0.5923&0.5829&0.6075&0.4847&0.6013&0.6182\\
&GADNR &0.6436&0.5943&0.5978&0.4554&\underline{0.6406}&0.6158&0.5913\\
&ADA-GAD &0.2034&0.4416&0.5765&0.4400&0.5369&0.4067&0.4342\\
&GGAD   &\underline{0.9443}&0.8228&\underline{0.6354}&0.6180&0.5340&0.6476&0.7003\\
&RHO &0.9302&\textbf{0.8623}&0.6207&0.5914&0.6255&\underline{0.7129}&\underline{0.7238}\\
&GraphNC   & \textbf{0.9613} & \underline{0.8340} & \textbf{0.6420} & \textbf{0.6630} & \textbf{0.6505} & \textbf{0.7693} & \textbf{0.7533} \\
\hline
\hline
\multirow{10}{*}{AUPRC}
&DOMINANT &0.7289&0.0542&0.0414&0.2237&0.2217&0.1283&0.2330\\
&AnomalyDAE    &0.7748&0.0538&0.0362&0.2064&0.2697&0.1177&0.2431\\
&OCGNN   &0.6895&0.0492&0.0400&0.1937&0.2138&0.1501&0.2227\\

&AEGIS     &0.2616&0.0685&0.0441&0.2192&0.1943&0.1110&0.1497\\
&GAAN    &0.0856&0.0324&0.0362&0.1724&0.1693&0.0768&0.0945\\
&TAM    &0.5183&0.0551&0.0446&0.2184&0.2178&0.1087&0.1938\\
&GADNR &0.0866&0.0523&0.0456&0.1324&\textbf{0.3511}&0.1430&0.1352\\
& ADA-GAD &0.0388&0.0375&0.0417&0.1295&0.2566&0.0723&0.0961\\

&GGAD   &\underline{0.7922}&0.1825&0.0520&\underline{0.2261}&0.2502&0.1442&0.2745\\
&RHO &0.7879&\textbf{0.4893}&\textbf{0.0616}&0.1918&\underline{0.3256}&\underline{0.2337}&\underline{0.3483}\\
&GraphNC   & \textbf{0.8403} & \underline{0.3667} & \underline{0.0560} & \textbf{0.2389} & 0.3082 & \textbf{0.3561} & \textbf{0.3610} \\
\hline 
\hline
\end{tabular}
}
\label{tab:mainR}
\end{center}
\end{table*}

\section{Experiments}
\noindent \textbf{Datasets.}
We evaluate the effectiveness of the GraphNC across six real-world GAD datasets drawn from diverse domains, including social networks: Reddit \citep{kumar2019predicting}, online shopping co-review networks: Amazon \citep{dou2020enhancing} and YelpChi \citep{dou2020enhancing}, a co-purchase network: Photo \citep{shchur2018pitfalls}, a collaboration network: Tolokers \citep{mcauley2015image}, and financial networks: T-Finance \citep{tang2022rethinking}. Additional description and statistical results are provided in App. \ref{app:dataset}. 

\noindent \textbf{Competing Methods.}
The competing methods include reconstruction-based methods: DOMINANT \citep{ding2019deep}, GADNR \citep{roy2024gad}, ADA-GAD \citep{he2024ada} and AnomalyDAE \citep{fan2020anomalydae}, one-class classification: OCGNN \citep{wang2021one} and RHO \citep{ai2026semi}, adversarial learning based methods: AEGIS \citep{ding2021inductive} and GAAN \citep{chen2020generative},  affinity maximization-based methods: TAM \citep{hezhe2023truncated}, and anomaly generation-based method: GGAD \citep{qiao2024generative}. A more detailed description of competing methods can be found in App. \ref{app:competing_methods}. Except  GGAD and RHO that are specifically designed for the semi-supervised setting, the rest of competing methods are primarily unsupervised methods, which are adapted to our semi-supervised setting following \citep{qiao2024generative}: reconstruction is performed using only the normal nodes, the one-class center calculation is based on only the normal nodes, the adversarial discriminator is trained with the normal nodes, and the affinity maximization is also applied on the normal nodes solely. Note that some unsupervised methods, such as CoLA \citep{liu2021anomaly}, GRADATE \citep{duan2023graph}, and HUGE \citep{pan2025label}, cannot be adapted to semi-supervised setting due to their design more relying on fully unlabeled data, which are excluded from our comparison.

\noindent \textbf{Evaluation Metric.}
Following \citep{liu2021anomaly, qiao2024generative, ai2026semi}, two widely-used metrics, AUROC and AUPRC, are used to evaluate the performance of all methods. For both metrics, a higher
value denotes a better performance. Moreover, for each method, we report the average performance after 5 independent runs with different random seeds.

\noindent \textbf{Implementation Details.} \label{Sec:Implementation}
Our GraphNC model is implemented in Pytorch 1.10.0 with Python 3.8, and
all the experiments are performed using GeForce RTX 3090 (24GB). GraphNC is optimized using the Adam optimizer \citep{kinga2015method} with a learning rate of $5e-3$ for Photo and Reddit that have relatively high attribute dimensions, and $5e-4$ for Amazon, T-Finance, YelpChi, and Tolokers that have relatively low attribute dimensions, to avoid overfitting. The pre-trained teacher is set to GGAD \citep{qiao2024generative} by default, the experimental results of other teacher models can be found in  Sec. \ref{sec:teacher}. $\alpha$ that controls the importance of NormReg is set to $0.01$ and the mask ratio for the augment mechanism is set $0.30$  by default.

\begin{table*}[t]
\caption{Results of GraphNC using different teacher models. `(\texttt{MODEL NAME})' denotes the teacher model used.}
\label{gen_auroc}
\begin{center}
\resizebox{0.95\textwidth}{!}{
\begin{tabular}{c|c|cccccc|c}
\hline
\hline
\multirow{2}*{\textbf{Metric}}&\multirow{2}*{\textbf{Method}} & \multicolumn{6}{c|}{\textbf{Dataset}}\\
&&Amazon &T-Finance &Reddit &YelpChi &Tolokers &Photo &Avg.\\
\hline
\multirow{6}{*}{AUROC}
&DOMINANT &0.8867&0.6167&0.5194&0.6314&0.5121&0.5314&0.6162\\
&GraphNC (DOMINANT) &\textbf{0.8936}&\textbf{0.7985}&\textbf{0.5260}&\textbf{0.6543}&\textbf{0.6977}&\textbf{0.6829}&\textbf{0.7088}\\
\cline{2-9}
&OCGNN   &0.8810&0.5742&0.5622&0.6155&0.4803&0.6461&0.6265\\
&GraphNC (OCGNN)   &\textbf{0.9241}&\textbf{0.6819}&\textbf{0.6093}&\textbf{0.6447}&\textbf{0.6069}&\textbf{0.7068}&\textbf{0.6956}\\
\cline{2-9}
&GGAD   &0.9443&0.8228&0.6354&0.6180&0.5340&0.6476&0.7003\\
&GraphNC (GGAD)   & \textbf{0.9613} & \textbf{0.8340} & \textbf{0.6420} & \textbf{0.6630} & \textbf{0.6505} & \textbf{0.7693} & \textbf{0.7533} \\
\hline
\hline
\multirow{6}{*}{AUPRC}
&DOMINANT &0.7289&0.0542&\textbf{0.0414}&0.2237&0.2217&0.1283&0.2330\\
&GraphNC (DOMINANT)  &\textbf{0.7290}&\textbf{0.1703}&\textbf{0.0414}&\textbf{0.2375}&\textbf{0.3465}&\textbf{0.1526}&\textbf{0.2795}\\
\cline{2-9}
&OCGNN   &0.6895&0.0492&0.0400&0.1937&0.2138&0.1501&0.2227\\
&GraphNC (OCGNN)   &\textbf{0.7372}&\textbf{0.0748}&\textbf{0.0512}&\textbf{0.2116}&\textbf{0.2814}&\textbf{0.1823}&\textbf{0.2564}\\
\cline{2-9}
&GGAD   &0.7922&0.1825&0.0520&0.2261&0.2502&0.1442&0.2745\\
&GraphNC (GGAD)   & \textbf{0.8403} & \textbf{0.3667} & \textbf{0.0560} & \textbf{0.2389} & \textbf{0.3082} & \textbf{0.3561} & \textbf{0.3610} \\
\hline 
\hline
\end{tabular}
}
\label{tab:learn}
\end{center}
\end{table*}

\begin{table}[t]
\caption{AUROC and AUPRC results of GraphNC and its variants.} 
\label{albations}
\begin{center}
\resizebox{1.04\columnwidth}{!}{
\begin{tabular}{l|ccc}
\toprule
\multirow{3}*{\textbf{Method}}& \multicolumn{3}{c}{\textbf{Dataset}} \\
&  Amazon & T-Finance & Tolokers\\
\midrule
& \multicolumn{3}{c}{\textbf{AUROC}}\\
\cline{2-4}
OT &0.9443&0.8228&0.5340\\
OT+NormReg  & 0.8956 & \textbf{0.8455} & 0.5843 \\
OT+NormReg-Finetune  & 0.9060 & 0.8008 & 0.5396 \\
OT+ScoreDA  & \underline{0.9511} & 0.8300 & 0.5811 \\
OT+ScoreDA+NormReg*   & 0.9452& 0.8278 & \underline{0.6492} \\
OT+ScoreDA+NormReg  & \textbf{0.9613} & \underline{0.8340} & \textbf{0.6505} \\
\midrule
& \multicolumn{3}{c}{\textbf{AUPRC}}\\
\cline{2-4}
OT &0.7922&0.1825&0.2502\\
OT+NormReg & 0.6568 & 0.2413 & 0.2688 \\
OT+NormReg-Finetune  & 0.7008 & 0.1406 & 0.2465 \\
OT+ScoreDA  & \underline{0.8189} & \underline{0.2566} & 0.2567 \\
OT+ScoreDA+NormReg*  & 0.8180 & 0.1971 & \underline{0.3034} \\
OT+ScoreDA+NormReg  & \textbf{0.8403} & \textbf{0.3667} & \textbf{0.3082} \\
\bottomrule
\end{tabular}
}
\vspace{-1em}
\label{tab:Ab}
\end{center}
\end{table}

\subsection{Main Experiments}
The main comparison results are shown in Table \ref{tab:mainR}, where all models use 15\% labeled normal nodes during training.  From the results, we find that GGAD surpasses adapted reconstruction-based, one-class classification, and adversarial learning approaches, suggesting that the classifier based on anomaly generation is particularly effective for learning the normality in semi-supervised GAD. 
By learning adaptive spectral filters, RHO mitigates the limited generalization of normal nodes under varying homophily levels, thereby achieving the second-best performance.
Despite building upon GGAD, GraphNC consistently outperforms its pretrained teacher, GGAD, and the best competing method, RHO, having an average 4.58\% AUROC and 7.35\% AUPRC improvement over RHO. These experimental results indicate that the two components in GraphNC can effectively calibrate the normality learned from GGAD by aligning its scores with the output of GGAD, while at the same time refining itself through representation regularization to mitigate inaccurate anomaly scores from GGAD. Besides, although the reconstruction-based, one-class classification-based, and affinity maximized methods achieve competitive performance on Amazon, T-Finance, and YelpChi, they still largely underperform GraphNC. This is primarily because their modeling of normal patterns is less effective than that of the anomaly-generation-based method GGAD.

\subsection{GraphNC Enabled Teacher Models} \label{sec:teacher}
To demonstrate that GraphNC is flexible and can be plugged into various teacher models, 
in addition to the default teacher model GGAD, 
we further evaluate the performance GraphNC when plugging the reconstruction-based method DOMINANT and the one-class classification-based method OCGNN as a pre-trained teacher model into our model. The AUROC and AUPRC results are shown in Table \ref{tab:learn}. It is clear from the results that, similar to the results using GGAD as the teacher model, GraphNC can substantially and consistently improve the performance of DOMINANT and OCGNN in terms of both AUROC and AUPRC. In particular, on Tolokers and Photo, both teacher models struggle to deliver strong performance, whereas applying GraphNC leads to substantial gains in both AUROC and AUPRC, resulting in a 36.2\% and 26.3\%  AUROC improvement on Tolokers for DOMINANT and OCGNN, respectively. The main reason is that aligning the score distribution from the pre-trained teacher model in GraphNC enables the student to calibrate the coarse normality learned by the teacher. This is because that GraphNC can not only inherit the strengths of the teacher but also correct its inaccuracies and refines normality on its own, resulting in consistent performance improvements. Importantly, given a better teacher model, GraphNC can obtain better GAD performance, \eg, GraphNC (GGAD) vs. GraphNC (DOMINANT) and GraphNC (OCGNN).

\begin{table*}[ht]
\centering
\caption{AUROC and AUPRC results of GraphNC and its variants using different teacher models.}
\label{tab:ablation_all}
\resizebox{0.95\textwidth}{!}{
\begin{tabular}{l|l| cccccc|c}
\toprule
\toprule
Metric & \ \ \ \ \ \ \ \ Method & Amazon & T\text{-}Finance & Reddit & YelpChi & Tolokers & Photo & Avg. \\
\midrule
\multirow{13}{*}[-0.4ex]{AUROC}
&\multicolumn{8}{c}{DOMINANT}\\
\cmidrule(l{2pt}r{2pt}){2-9}
& OT+ScoreDA+NormReg       & \textbf{0.8936} & \textbf{0.7985} & \underline{0.5260} & \textbf{0.6543} & \textbf{0.6977} & \underline{0.6829} & \textbf{0.7088} \\
& OT+ScoreDA+NormReg* & 0.8674 & 0.4301 & \underline{0.5260} & 0.5382 & \underline{0.6895} & \textbf{0.6902} & \underline{0.6236} \\
& OT+ScoreDA   & 0.7103 & \underline{0.7069} & 0.4623 & 0.5125 & 0.5342 & 0.6384 & 0.5941 \\
& OT+NormReg-Finetune    & 0.6054 & 0.4920 & 0.5192 & 0.6299 & 0.6258 & 0.5188 & 0.5651 \\
& OT+NormReg & 0.7320 & 0.5041 & \textbf{0.5968} & \underline{0.6315} & 0.5943 & 0.5128 & 0.5952 \\
& OT      & \underline{0.8867} & 0.6167 & 0.5194 & 0.6314 & 0.5121 & 0.5314 & 0.6162 \\
\cmidrule(l{2pt}r{2pt}){2-9}
&\multicolumn{8}{c}{OCGNN}\\
\cmidrule(l{2pt}r{2pt}){2-9}
& OT+ScoreDA+NormReg          & \underline{0.9255} & \textbf{0.6819} & \textbf{0.6093} & \textbf{0.6447} & 0.6069 & \textbf{0.7068} & \textbf{0.6956} \\
& OT+ScoreDA+NormReg* & \textbf{0.9276} & 0.4101 & 0.5507 & 0.6397 & \textbf{0.6321} & 0.6462 & 0.6344 \\
& OT+ScoreDA & 0.9158 & 0.4101 & 0.5235 & \underline{0.6411} & \underline{0.6134} & \underline{0.6913} & 0.6325 \\
& OT+NormReg-Finetune          & 0.8274 & 0.5960 & 0.4939 & 0.4738 & 0.5309 & 0.5588 & 0.5801 \\
& OT+NormReg  & 0.9095 & \underline{0.6113} & \underline{0.5820} & 0.6248 & 0.5245 & 0.6478 & \underline{0.6499} \\
& OT               & 0.8810 & 0.5742 & 0.5622 & 0.6155 & 0.4803 & 0.6461 & 0.6265 \\
\midrule
\multirow{13}{*}[-0.4ex]{AUPRC}
&\multicolumn{8}{c}{DOMINANT}\\
\cmidrule(l{2pt}r{2pt}){2-9}
& OT+ScoreDA+NormReg       & \textbf{0.7290} & \textbf{0.3703} & 0.0414 & \textbf{0.2375} & \underline{0.3465} & \underline{0.1526} & \textbf{0.2795} \\
& OT+ScoreDA+NormReg* & 0.5398 & 0.0363 & \underline{0.0415} & 0.1630 & \textbf{0.3468} & \textbf{0.1554} & 0.2138 \\
& OT+ScoreDA   & 0.1466 & \underline{0.1759} & 0.0312 & 0.1581 & 0.2445 & 0.1274 & 0.1472 \\
& OT+NormReg-Finetune          & 0.0750 & 0.0415 & 0.0330 & 0.2292 & \underline{0.2793} & 0.1122 & 0.1283 \\
& OT+NormReg  & 0.1094 & 0.0424 & \textbf{0.0486} & \underline{0.2307} & 0.2445 & 0.1066 & 0.1303 \\
& OT             & \underline{0.7289} & 0.0542 & 0.0414 & 0.2237 & 0.2217 & 0.1283 & \underline{0.2330} \\
\cmidrule(l{2pt}r{2pt}){2-9}
&\multicolumn{8}{c}{OCGNN}\\
\cmidrule(l{2pt}r{2pt}){2-9}
& OT+ScoreDA+NormReg          & \textbf{0.7372} & \textbf{0.0748} & \textbf{0.0512} & \textbf{0.2116} & 0.2814 & \textbf{0.1823} & \textbf{0.2564} \\
& OT+ScoreDA+NormReg* & 0.6106 & 0.0350 & 0.0358 & 0.2048 & \underline{0.2983} & 0.1281 & 0.2188 \\
& OT+ScoreDA   & 0.6298 & 0.0350 & 0.0327 & \underline{0.2112} & \textbf{0.2915} & \underline{0.1782} & 0.2297 \\
& OT+NormReg-Finetune          & 0.4785 & 0.0554 & 0.0326 & 0.1382 & 0.2433 & 0.1093 & 0.1762 \\
& OT+NormReg  & \underline{0.7015} & \underline{0.0567} & \underline{0.0444} & 0.2036 & 0.2415 & 0.1525 & \underline{0.2333} \\
& OT               & 0.6895& 0.0492 & 0.0400 & 0.1937 & 0.2138 & 0.1501 & 0.2227 \\
\bottomrule
\bottomrule
\end{tabular}}
\end{table*}

\subsection{Ablation Study with the Default Teacher Model} \label{app:ablation_s}
Ablation study is perform to evaluate individual components in GraphNC, with GGAD as the default teacher model. To this end, several variants of GraphNC are introduced. (1) Only Teacher (\textbf{OT}) directly uses the pre-trained semi-supervised teacher model to derive the anomaly score. (2) \textbf{OT+NormReg} incorporates NormReg into the pre-training of the teacher model. (3) \textbf{OT+NormReg+Finetune} leverages NormReg to fine-tune the pre-trained teacher. (4) \textbf{OT+ScoreDA} includes the student model and adds ScoreDA on top of OT (\ie, optimizing GraphNC using ScoreDA only). (5) \textbf{OT+ScoreDA+NormReg*} applies NormReg on top of OT+ScoreDA, but it applies the NormReg loss on all nodes instead of the labeled normal nodes only. \textbf{OT+ScoreDA+NormReg} is the default full model using both ScoreDA and NormReg. 

The experimental results are shown in Table \ref{tab:Ab}. The ablation study experiments on all datasets are provided in App. \ref{app:ablation_all_data}. We observe that OT+ScoreDA enhances the performance of OT, achieving the second-best results on most datasets in terms of AUPRC. This demonstrates that the ScoreDA component can effectively calibrate the normality learned from the teacher, ensuring better separation 
between normal and abnormal classes. However, OT+ScoreDA underperforms the full model (the last row) on all datasets, highlighting that the negative impact of inaccurate anomaly scores from the teacher model can be effectively mitigated by the NormReg component. On the other hand, applying NormReg directly to the teacher via either joint optimization (OT+NormReg) or finetuning (OT+NormReg+Finetune) can improve OT to some extent. This justifies the effectiveness of NormReg from a different perspective, but NormReg works best when working in the student model and combining with ScoreDA due to its joint effects in minimizing the detection errors discussed in our theoretical analysis. 

Meanwhile, we evaluate the performance of GraphNC when the teacher performs poorly, the average performance improvement of GraphNC w.r.t. different accuracy levels, and the comparison of the number of true anomalous instances that are accurately detected for GGAD and GraphNC. All the results can be found in App. \ref{app:Additional}.

\begin{figure}[t]
  \centering
  \includegraphics[width=0.9\columnwidth]{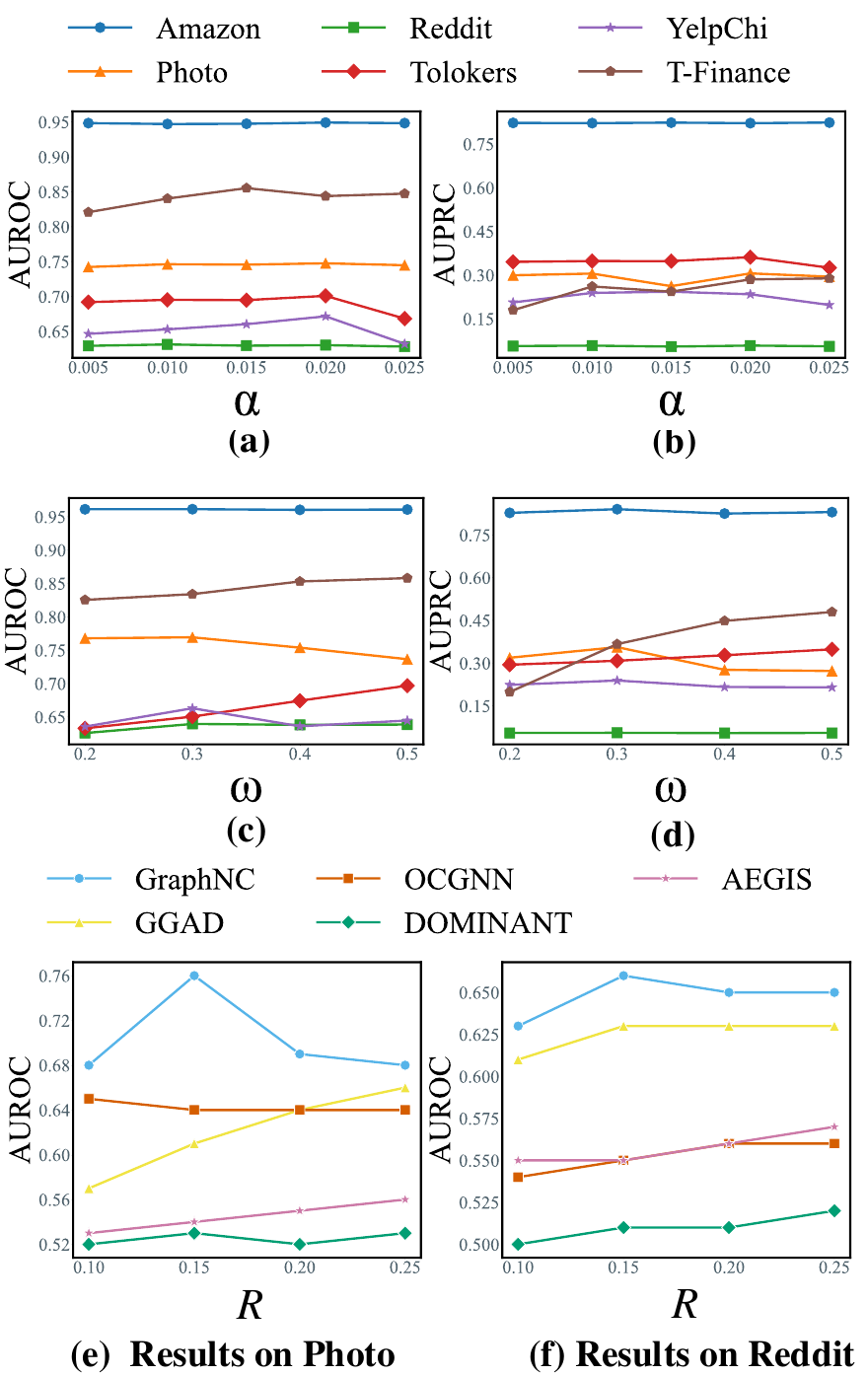}
  \caption{(a-d) AUROC and AURPC results w.r.t $\alpha$ and $\omega$. (e-f) AUROC results w.r.t $R$ on Photo and Reddit.}
  \label{fig:sensitive}
\end{figure}

\subsection{Ablation Study with Other Teacher Models} \label{app:ablation_all}
We also conduct the ablation study of the GraphNC-enabled framework on two other types of teacher models to demonstrate the effectiveness of each module in the GraphNC. We evaluate either the reconstruction-based method, DOMINANT, or the one-class classification-based method, OCGNN, as the teacher model and make a comparison with the variants of GraphNC. 

The experimental results are shown in Table \ref{tab:ablation_all}.
We observe that, similar to the results in Table \ref{tab:Ab}, GraphNC outperforms the corresponding variants on most of the datasets for both DOMINANT and OCGNN. This further demonstrates the flexibility and effectiveness of GraphNC in enhancing existing semi-supervised methods.

\subsection{Hyperparameter Sensitivity Analysis}
We evaluate the sensitivity of GraphNC w.r.t. the loss moderator $\alpha$,  mask ratio $\omega$, and training size $R$.
Detailed analyses of the effects of teachers' performance, the augmented approach, threshold selection for FP/FN computation,  time complexity, and running time comparison are provided in App. \ref{app:Additional} and \ref{app:time_complexity}.

\noindent \textbf{Performance w.r.t. the Loss Moderator $\alpha$.} 
As shown in Figs. \ref{fig:sensitive} (a-b), our model GraphNC remains generally stable as the consistency regularization weight
$\alpha$ varies. However, on Tolokers and YelpChi, the performance 
slightly declines as $\alpha$ increases. This is mainly because excessive emphasis on the consistency may lead to an overly compressed representation space, which may in turn weaken the representation discriminability.


\noindent \textbf{Performance w.r.t. the Mask Ratio $\omega$.} We evaluate GraphNC under different masking ratios. As shown in Figs. \ref{fig:sensitive} (c-d), we observe that increasing the masking ratio $\omega$ generally enhances the performance of
GraphNC on Tolokers and T-Finance, while leading to performance degradation on Photo and YelpChi. This suggests that different datasets benefit from different levels of augmentation, depending on the variation in the underlying normality. For consistency, we set $\omega$ to $0.3$ as the default value across all datasets.

\noindent \textbf{Performance w.r.t. the Training Size $R$. }
To explore the impact of training size $R$, we compare GraphNC with four semi-supervised GAD methods using varying numbers of training normal nodes, as shown in Figs.\ref{fig:sensitive} (e-f). Similar with other competing methods, with the increase of training size $R$, our model GraphNC maintains the superiority over other competing methods, demonstrating the effectiveness of GraphNC using different scales of labeled training data. 


\section{Conclusion}
In this paper, we propose GraphNC, a generic graph normality calibration framework designed to enhance existing semi-supervised GAD methods, achieving substantial and consistent performance improvements across different types of existing approaches. GraphNC is composed of two key components: ScoreDA and NormReg. The ScoreDA module aligns the anomaly scores with the anomaly score distribution produced by the teacher model, effectively pushing the anomaly scores for the normal and abnormal classes toward opposite ends. This alignment ensures a clearer separation between normal and abnormal nodes in the score space, thereby leading to a lower FPR and FNR.
The NormReg module is introduced to reduce the impact of inaccurate scores in the anomaly score distribution served by the teacher, enabling GraphNC to learn more compact normal representations.  
GraphNC is comprehensively verified by empirical results on six GAD datasets, which show that it consistently enhances performance of three different types of teacher models, and it achieves the best performance.


\noindent \textbf{Limitation and Future Work.}
Our empirical results suggest that GraphNC can still yield consistent improvements even when the teacher performs relatively poorly, because of the role of ScoreDA and NormReg. Nevertheless, the effectiveness of ScoreDA still depends on the teacher being able to offer at least coarse supervisory information. Future work will investigate additional forms of supplementary supervision and some expert knowledge to further reduce the framework’s dependence on the teacher model.

\section*{Acknowledgments}
This research is supported in part by the National Research Foundation, Singapore and CyberSG R\&D Programme Office under its Translation and Innovation Grants (CRPO-GC4-SMU-001), the Lee Kong Chian Fellowship (T050273), A*STAR under its MTC YIRG Grant (M24N8c0103), and the Singapore Ministry of Education (MOE) Academic Research Fund (AcRF) Tier 1 Grant (24-SIS-SMU-008).


\nocite{langley00}
\section*{Impact Statement}
This paper presents work whose goal is to advance the field of machine learning. There are many potential societal consequences of our work, none of which we feel must be specifically highlighted here.

\bibliography{example_paper}
\bibliographystyle{icml2026}

\newpage
\appendix
\onecolumn
\section{Detailed Description of Datasets} \label{app:dataset}
The key statistics of the datasets are presented in Table \ref{datastets-table}. A detailed introduction of these datasets is given as follows.
\begin{itemize}

\item Amazon \citep{dou2020enhancing}: It is a co-review network obtained from the Musical Instrument category on
Amazon.com. There are also three relations: U-P-U (users reviewing at least one same
product), U-S-U (users having at least one same star rating within one week), and U-V-U
(users with top-5\% mutual review similarities).

\item T-Finance \citep{tang2022rethinking}: It is a financial transaction network where the node represents an anonymous
account and the edge represents two accounts that have transaction records. Some attributes
of logging like registration days, logging activities, and interaction frequency, etc, are used
as the features of each account. Users are labeled as anomalies if they fall into categories
such as fraud, money laundering, or online gambling

\item Reddit \citep{kumar2019predicting}: It is a user-subreddit graph which captures one month’s worth of posts shared
across various subreddits at Reddit. The node represents the users, and the text of each post
is transformed into a feature vector and the features of the user and subreddits are the feature
summation of the post they have posted. The anomalies are the used who have been banned
by the platform

\item YelpChi \citep{dou2020enhancing}: Similar to Amazon-all, YelpChi-all incorporates three types of edges: R-U-R (reviews written by the same user), R-S-R (reviews of the same product with identical star ratings), and R-T-R (reviews of the same product posted within the same month). YelpChi-all is then constructed by merging these different relations into a single unified relation \citep{chen2022gccad, hezhe2023truncated}.

\item Tolokers \citep{mcauley2015image}: It is obtained from the Toloka crowdsourcing platform, where the node
represents the person who has participated in the selected project, and the edge represents
two workers work on the same task. The attributes of the node are the profile and task
performance statistics of workers

\item Photo \citep{shchur2018pitfalls}: It is an Amazon co-purchase network where nodes represent products and edges indicate co-purchase relationships. Each node is described by a bag-of-words representation derived from user reviews
\end{itemize}

\section{Description of Baselines} \label{app:competing_methods}
A more detailed introduction of the nine competing GAD models is given as follows.
\begin{itemize}
\item DOMINANT \citep{ding2019deep}: It uses a three-layer GCN to reconstruct both structure and attributes, and the resulting reconstruction error on both structure and attribute is taken as the anomaly score.

\item AnomalyDAE \citep{fan2020anomalydae}:
It incorporates graph attention into the GNN-based structure and attribute encoders, where the anomaly score is derived from the reconstruction error.

\item OCGNN \citep{wang2021one}:  It applies a one-class approach on the GNN, aiming to combine the representational power of the GNN with the anomaly detection capability of one-class classification. The anomaly score is computed by measuring the distance of each point from the center.

\item AEGIS \citep{ding2021inductive}:  It leverages a generative adversarial network (GAN), where the generator is designed to produce pseudo anomalies, while the discriminator distinguishes between genuine normal nodes and the generated anomalies.

\item GAAN  \citep{chen2020generative}: It leverages a generative adversarial network where fake graph nodes are generated. Then the covariance matrix for real nodes and fake nodes are computed to enhace the node. Finally, a discriminator is trained to recognize whether two connected nodes are from a real or fake node.

\item TAM \citep{hezhe2023truncated}: It first reveals the 'one class homophily' and introduces the new anomaly measure score 'local node affinity'. As the local node affinity can not be directly obtained based on the embedding, it is optimized on truncated graphs where non-homophily edges are removed iteratively. The local node affinity is calculated on the learned representations on the truncated graphs.

\item GADNR \citep{roy2024gad}: It employs a GNN-based autoencoder to reconstruct each node’s attributes, degree, and Gaussian-approximated neighbor feature distribution, and uses the weighted sum of these reconstruction losses as the anomaly score.

\item ADA-GAD \citep{he2024ada}: It adopts a two-stage anomaly-denoised autoencoder framework that first pretrains on anomaly-reduced augmented graphs and then retrains decoders on the original graph, taking the final reconstruction error with a distribution regularization term as the anomaly score.

\item GGAD \citep{qiao2024generative}: It employs two priors related to anomalies, asymmetric local affinity and egocentric closeness, to generate pseudo anomaly nodes that can well simulate the real abnormal nodes. Then, a binary one-class classifier is trained based on the existing normal nodes and generated outliers for semi-supervised GAD.

\item RHO \citep{ai2026semi}: It adaptively learns robust homophily patterns for semi-supervised GAD, addressing the diverse homophily levels among normal nodes. Specifically, it employs adaptive frequency-response filters to capture heterogeneous normal patterns from both channel-wise and cross-channel views, and further aligns these representations via graph normality alignment for robust one-class anomaly detection.

\end{itemize}

\section{Additional Experimental Results} \label{app:Additional}

\begin{table}[]
\caption{Key statistics of GAD datasets.}
\centering
\label{datastets-table}
\scalebox{1.1}{
\renewcommand\tabcolsep{3pt}
\renewcommand{\arraystretch}{1.1}
\begin{tabular}{c|cccccc}
\toprule
Datasets & Amazon &T-Finance & Reddit &  YelpChi  &Tolokers   &Photo     \\\hline
\#Nodes & 11,944 &39,357   & 10,984&  45,941&  11,758   &7,484   \\
\#Edges & 4,398,392 & 21,222,543    & 168,016& 3,846,979 & 519,000  & 119,043   \\
\#Attributes  & 25    &10  & 64 & 32 & 10  & 745  \\
Anomaly & 9.5\%   & 4.6\%    & 3.3\%  &14.52\% & 21.8\%  &4.9\%  \\ 
 \bottomrule
\end{tabular}}
\end{table}

\subsection{Score Distribution Visualization on Other Teacher Models}
In this section, we visualized the score distributions of the reconstruction-based method DOMINANT and the one-class classification method OCGNN, along with their corresponding GraphNC-enabled models on Amazon. As shown in the Fig. \ref{fig:example_all1} \& \ref{fig:example_all2}, we observe that the GraphNC-enabled model achieves better separation between normal and abnormal classes, further demonstrating that GraphNC can substantially and consistently enhance performance.

   
\begin{figure}[t]
\centering
\begin{subfigure}[t]{0.5\textwidth}
  \centering
  \includegraphics[width=0.7\linewidth]{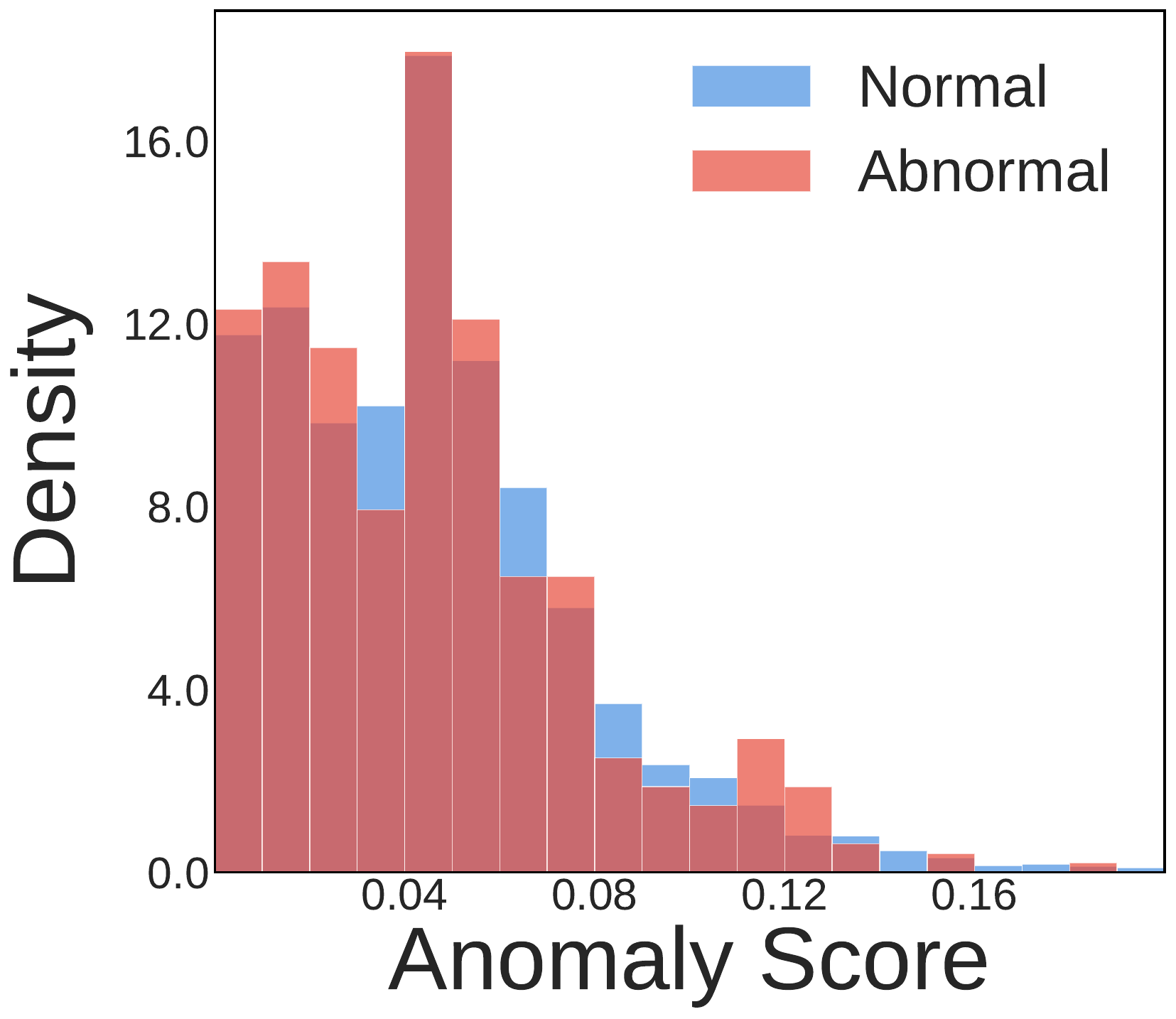}
  \caption{DOMINANT}
\end{subfigure}\hfill
\begin{subfigure}[t]{0.5\textwidth}
  \centering
  \includegraphics[width=0.7\linewidth]{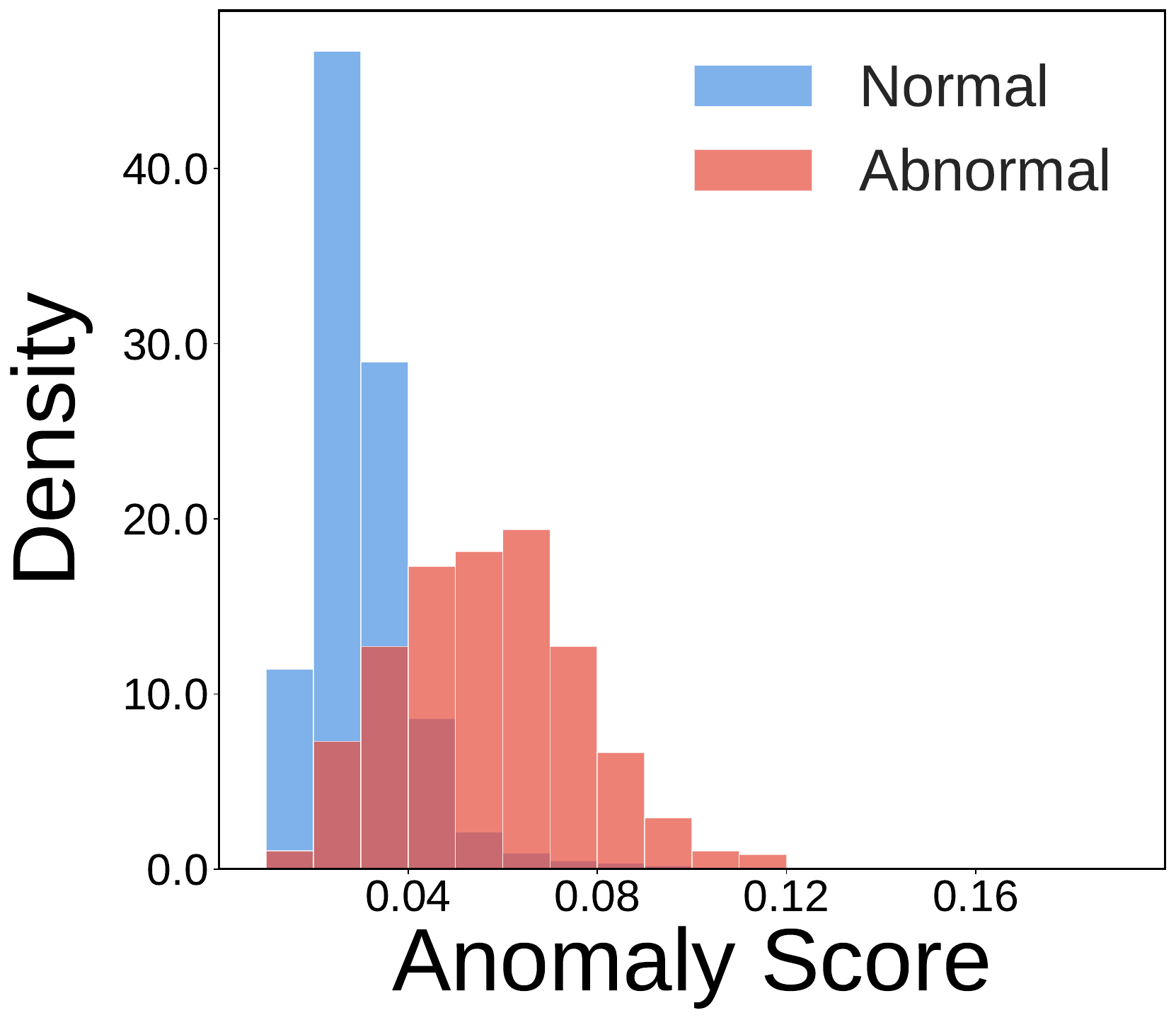}
  \caption{GraphNC(DOMINANT)}
\end{subfigure}
\caption{The score distribution of DOMINANT along the corresponding GraphNC enabled DOMINANT on Amazon \protect\citep{dou2020enhancing}.}
\label{fig:example_all1}
\end{figure}

\begin{figure}[t]
\centering
\begin{subfigure}[t]{0.5\textwidth}
  \centering
  \includegraphics[width=0.7\linewidth]{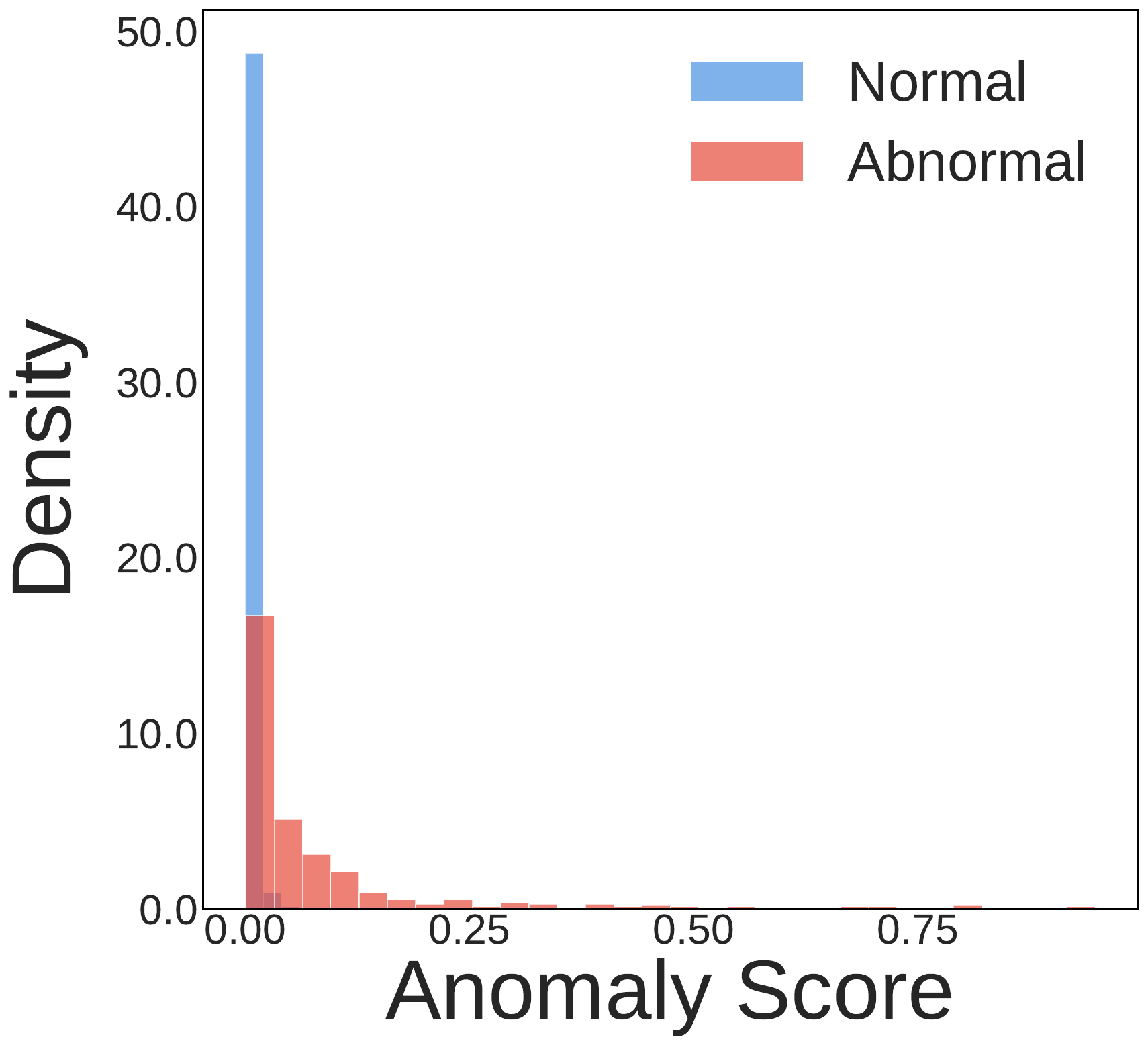}
  \caption{OCGNN}
\end{subfigure}\hfill
\begin{subfigure}[t]{0.5\textwidth}
  \centering
  \includegraphics[width=0.7\linewidth]{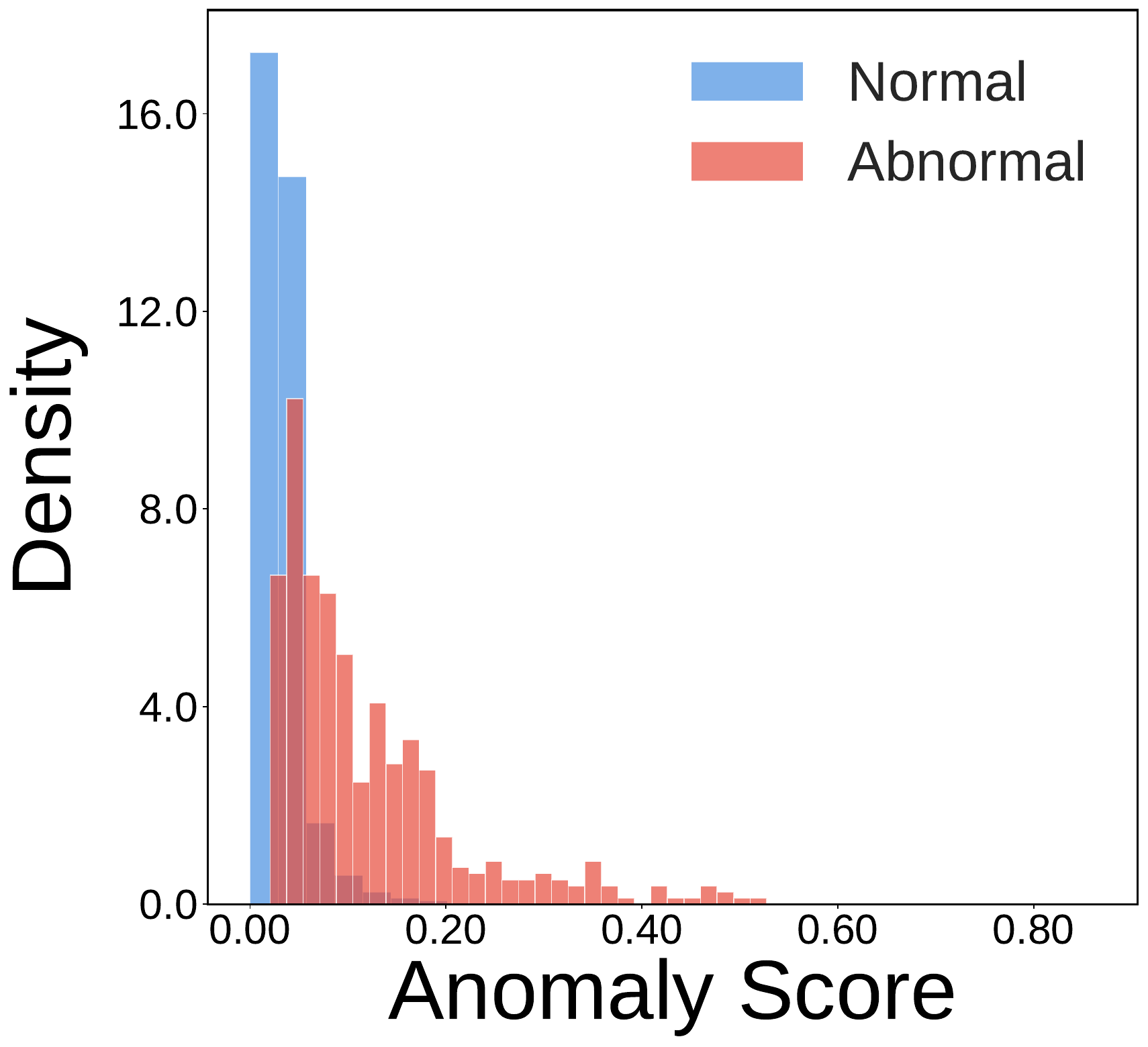}
  \caption{GraphNC(OCGNN)}
\end{subfigure}
\caption{The score distribution of OCGNN along the corresponding GraphNC enabled OCGNN on Amazon \protect\citep{dou2020enhancing}.}
\label{fig:example_all2}
\end{figure}

\begin{table*}[t]
\caption{AUROC and AUPRC results comparison of GraphNC and its five variants.} 
\label{albations}
\begin{center}
\resizebox{1.0\textwidth}{!}{
\begin{tabular}{c|l|cccccc|c}
\toprule
\hline
Metric & \ \ \ \ Method & Amazon & T-Finance & Reddit & YelpChi & Tolokers & Photo & Avg.\\
\midrule
\multirow{6}{*}{AUROC}
& OT &0.9443&0.8228&0.6354&0.6180&0.5340&0.6476&0.7003\\
& OT+NormReg  & 0.8956 & \textbf{0.8455} & 0.6317 & 0.6050 & 0.5843 & 0.6731 & 0.7058 \\
& OT+NormReg-Finetune  & 0.9060 & 0.8008 & 0.6089 & 0.5729 & 0.5396 & 0.6307 & 0.6764 \\
& OT+ScoreDA  & \underline{0.9511} & 0.8300 & \underline{0.6364} & 0.6263 & 0.5811 & 0.7397& 0.7274 \\
& OT+ScoreDA+NormReg*   & 0.9452& 0.8278 & 0.6212 & \underline{0.6271} & \underline{0.6492} & \underline{0.7433} & \underline{0.7356} \\
& OT+ScoreDA+NormReg  & \textbf{0.9613} & \underline{0.8340} & \textbf{0.6420} & \textbf{0.6630} & \textbf{0.6505} & \textbf{0.7693} & \textbf{0.7533} \\
\midrule
\multirow{6}{*}{AUPRC}
& OT &0.7922&0.1825&0.0520&\underline{0.2261}&0.2502&0.1442&0.2745\\
& OT+NormReg & 0.6568 & 0.2413 & 0.0520 & 0.2108 & 0.2688 & 0.1649 & 0.2657 \\
& OT+NormReg-Finetune  & 0.7008 & 0.1406 & 0.0454 & 0.1820 & 0.2465 & 0.1310 & 0.2410 \\
& OT+ScoreDA  & \underline{0.8189} & \underline{0.2566} & \underline{0.0530} & 0.2111 & 0.2567 & \underline{0.2969} & \underline{0.3155} \\
& OT+ScoreDA+NormReg*  & 0.8180 & 0.1971 & 0.0523 & 0.1866 & \underline{0.3034} & 0.2697 & 0.3045 \\
& OT+ScoreDA+NormReg  & \textbf{0.8403} & \textbf{0.3667} & \textbf{0.0560} & \textbf{0.2389} & \textbf{0.3082} & \textbf{0.3561} & \textbf{0.3610} \\
\hline

\hline
\bottomrule
\end{tabular}}
\label{tab:ablation}
\end{center}
\end{table*}


\subsection{Ablation Study on All Datasets} \label{app:ablation_all_data}
In this section, we perform an ablation study to evaluate the contribution of each component in GraphNC. To this end, several variants of GraphNC are introduced. (1) Only Teacher (\textbf{OT}) directly uses the pre-trained semi-supervised teacher model to derive the anomaly score. (2) \textbf{OT+NormReg} incorporates NormReg into the pre-training of the teacher model. (3) \textbf{OT+NormReg+Finetune} leverages NormReg to fine-tune the pre-trained teacher. (4) \textbf{OT+ScoreDA} includes the student model and adds ScoreDA on top of OT (\ie, optimizing GraphNC using ScoreDA only). (5) \textbf{OT+ScoreDA+NormReg*} applies NormReg on top of OT+ScoreDA but it applies the NormReg loss on all nodes instead of the labeled normal nodes only. \textbf{OT+ScoreDA+NormReg} is the default full model using both ScoreDA and NormReg. The experimental results are shown in Table \ref{tab:ablation}. We observe that OT+ScoreDA enhances the performance of OT, achieving the second-best results on four datasets in terms of AUPRC. This demonstrates that the ScoreDA component can effectively calibrate the normality learned from the teacher, ensuring better separation 
between normal and abnormal classes. However, OT+ScoreDA underperforms the full model (the last row) on all datasets, highlighting that the negative impact of inaccurate anomaly scores from the teacher model can be effectively mitigated by the NormReg component. On the other hand, applying NormReg directly to the teacher via either joint optimization (OT+NormReg) or finetuning (OT+NormReg+Finetune) can improve OT to some extent. This justifies the effectiveness of NormReg from a different perspective. However, NormReg achieves the best performance when applied to the student model and combined with ScoreDA, as their joint effects help minimize the detection errors discussed in our theoretical analysis.

\subsection{The Performance of GraphNC When the Teacher Performs Poorly}
To evaluate the performance of GraphNC when the teacher performs poorly (AUROC $\leq$ 0.5), we inject anomalies into the labeled normal nodes to deliberately degrade the teacher’s performance.  When the teacher model performs poorly, we still apply score distribution alignment and NormReg to achieve normality calibration. The experimental results using GGAD as teacher model are shown in Table \ref{toxic_auc}. 

From the results, we observe that when the teacher's performance is poor, \textit{e.g.}, the AUROC is lower than 0.5 on some datasets, GraphNC can largely enhance the performance. The main reason is that the errors introduced by teacher distillation mainly affect a small portion of nodes, for example, some normal nodes may receive slightly higher scores (while most normal nodes still remain low), and some anomalies may obtain lower scores. However, NormReg relies on the explicitly labeled normal nodes, rather than the teacher’s predictions, to enforce the compactness of the normality cluster. Consequently, during the ScoreDA process, the normal cluster continues to be tightened, while the correctly predicted normal nodes guide the adjustment of other normal nodes. This helps correct some of the teacher’s mistakes, ensuring that GraphNC can still achieve performance gains even when the teacher performs poorly.

Moreover, as shown in Table \ref{tab:learn}, we observe that DOMINANT and OCGNN perform poorly on certain datasets. For instance, on Photo, DOMINANT exhibits notably low performance, suggesting that the teacher’s label distribution becomes unreliable during score alignment. GraphNC leverages NormReg to prevent overfitting to these inaccurate supervisory signals and to better exploit the limited yet accurate normal supervision, thereby improving overall performance. Finally, we summarize the average improvement of GraphNC over the teacher model across different teacher score ranges, as presented in Table \ref{statistic}.

\begin{table*}[t]
\caption{AUROC and AUPRC results of GraphNC and GGAD* (GGAD is degraded by injecting anomalies into the labeled normal set).}
\label{toxic_auc}
\begin{center}
\resizebox{0.85\textwidth}{!}{
{\begin{tabular}{c|c|cccccc|c}
\hline
\hline
\multirow{2}*{\textbf{Metric}}&\multirow{2}*{\textbf{Method}} & \multicolumn{6}{c|}{\textbf{Dataset}}\\
&&Amazon &T-Finance &Reddit &YelpChi &Tolokers &Photo &Avg.\\
\hline
\multirow{2}{*}{AUROC}
&GGAD*   &0.6336&0.5132&0.4483&0.5130&0.4504&0.4465&0.5008\\
&GraphNC   & \textbf{0.7769} & \textbf{0.7450} & \textbf{0.5303} & \textbf{0.5685} & \textbf{0.6291} & \textbf{0.5970} & \textbf{0.6411} \\
\hline
\hline
\multirow{2}{*}{AUPRC}
&GGAD*   &0.1171&0.0443&0.0276&0.1603&0.1946&0.0856&0.1049\\
&GraphNC   & \textbf{0.2274} & \textbf{0.0930} & \textbf{0.0375} & \textbf{0.1823} & \textbf{0.2930} & \textbf{0.1145} & \textbf{0.1580} \\
\hline 
\hline
\end{tabular}
}}
\label{tab:main}
\end{center}
\end{table*}

\subsection{Analysis on True Anomalous Instances Recognition}
To offer more insights into the dependency between our model and the teacher model, we provide statistics for the number of anomalous instances which are accurately detected by the teacher (GGAD) versus those detected by GraphNC in Fig. \ref{fig:trueAnomalous}. We can observe that the number of anomalous instances accurately detected by GraphNC is consistently and substantially higher than the teacher model across different datasets where the teacher model's performance varies largely. This indicates that GraphNC can go far beyond the supervision from the teacher model to detect anomalies that the teacher model missed. We attribute this capability to the NormReg component and its synergy with the ScoreDA component.

\begin{figure}[t]
  \centering
  \includegraphics[width=0.4\columnwidth]{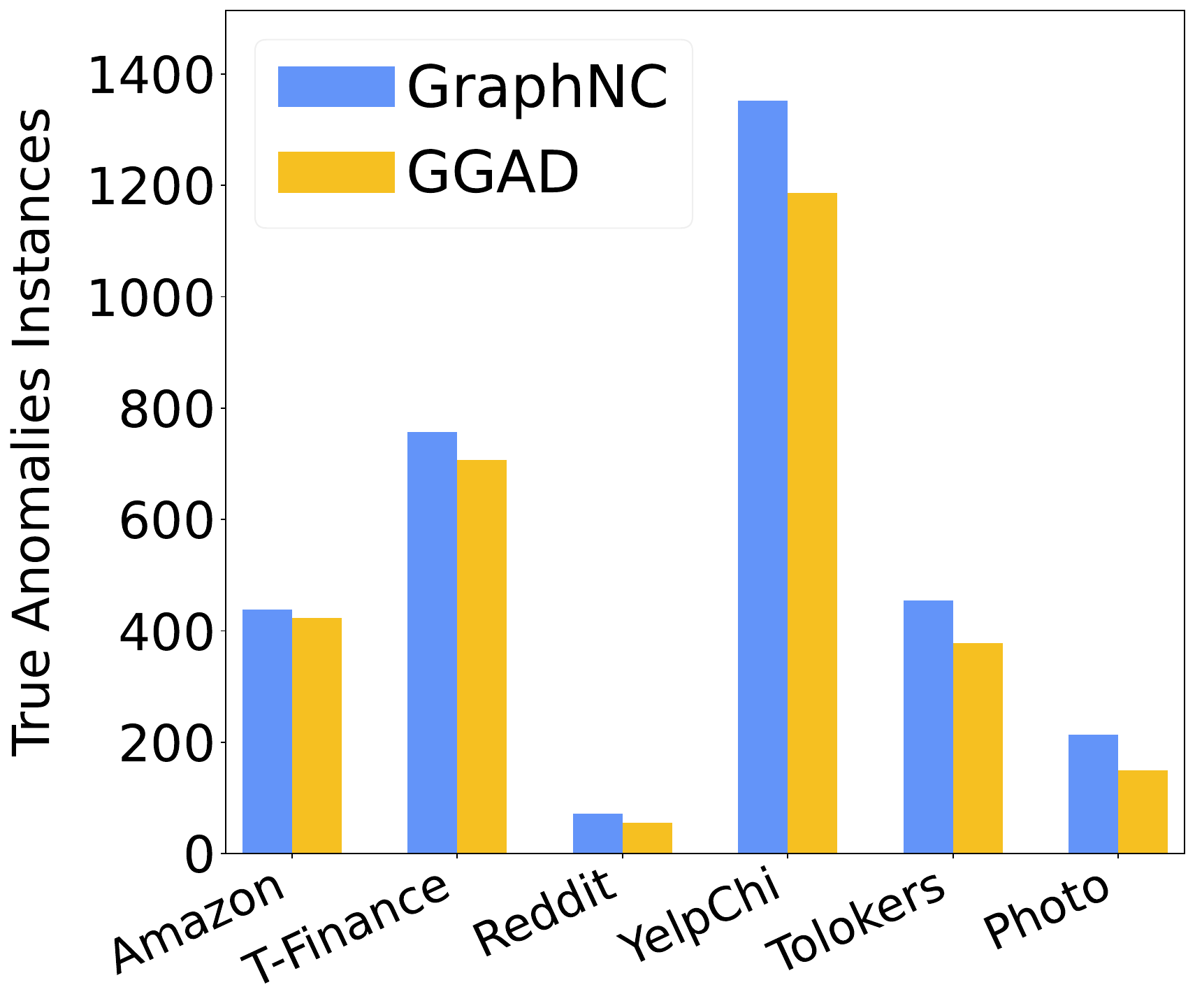}
  \caption{The number of true anomalous instances that are accurately detected by the teacher (GGAD) versus those detected by GraphNC.}
  \label{fig:trueAnomalous}
\end{figure}


\begin{table*}[t]
\caption{Statistics of the average AUROC improvement of GraphNC w.r.t. different accuracy levels of supervision from the teacher model GGAD. }
\label{statistic}
\begin{center}
{\resizebox{0.6\textwidth}{!}{
\begin{tabular}{c|cccc}
\hline
\hline
Teacher (GGAD) & (0.4$\sim$0.5) & (0.5$\sim$0.6) & (0.6$\sim$0.7) & (0.7$\sim$0.9) \\
\hline
GraphNC   & 0.1345 & 0.1128 & 0.0764 & 0.0204 \\
\hline
\hline
\end{tabular}
}}
\end{center}
\end{table*}
\subsection{Statistics of the Average AUROC Improvement of GraphNC w.r.t. Different Accuracy Levels}
We summarize the average improvement of GraphNC over the teacher model across different teacher score ranges, as presented in Table \ref{statistic}. From the results, we observe that when the teacher performs poorly, GraphNC achieves substantially larger improvements (e.g., $\Delta$(0.4$\sim$0.5) $>$ $\Delta$(0.7$\sim$0.9)). This is because the weaker teacher leaves more room for enhancement, and GraphNC is robust to different level of noisy supervision from the teacher model. This further demonstrates the robustness and effectiveness of GraphNC in normality calibration, even when the teacher performs poorly.

\subsection{GraphNC VS. GGAD with Two-stage Training}
To further demonstrate that the gain of GraphNC is not simply brought by the two-stage training procedure,
we design GGAD** as an additional baseline that  
replaced the student architecture and training objective (BCE) with those of GGAD for second-stage training.
Specifically, GGAD** is first trained on the labeled normal nodes and then used to generate pseudo-labels for unlabeled nodes. We then determine a threshold of 0.7 to separate normal and abnormal nodes, and apply the BCE loss for subsequent training. 

As shown in Fig. \ref{fig:variant} above, we observe that GraphNC consistently outperforms GGAD** on four datasets and in average performance. On T-Finance, simply incorporating the teacher loss slightly outperforms GraphNC. The main reason is that the chosen threshold may occasionally assist the model in identifying some anomalies. However, this improvement is not stable, as it is highly sensitive to threshold selection and the inaccuracy of predicted anomaly scores. Overall, Fig. \ref{fig:variant} and the comprehensive ablation study in Sec. \ref{app:ablation_s} and App. \ref{app:ablation_all_data} indicate that directly leveraging the teacher loss in the second stage is less effective than GraphNC.


\begin{figure}[t]
\centering
\begin{subfigure}[t]{0.5\textwidth}
  \centering
  \includegraphics[width=0.7\linewidth]{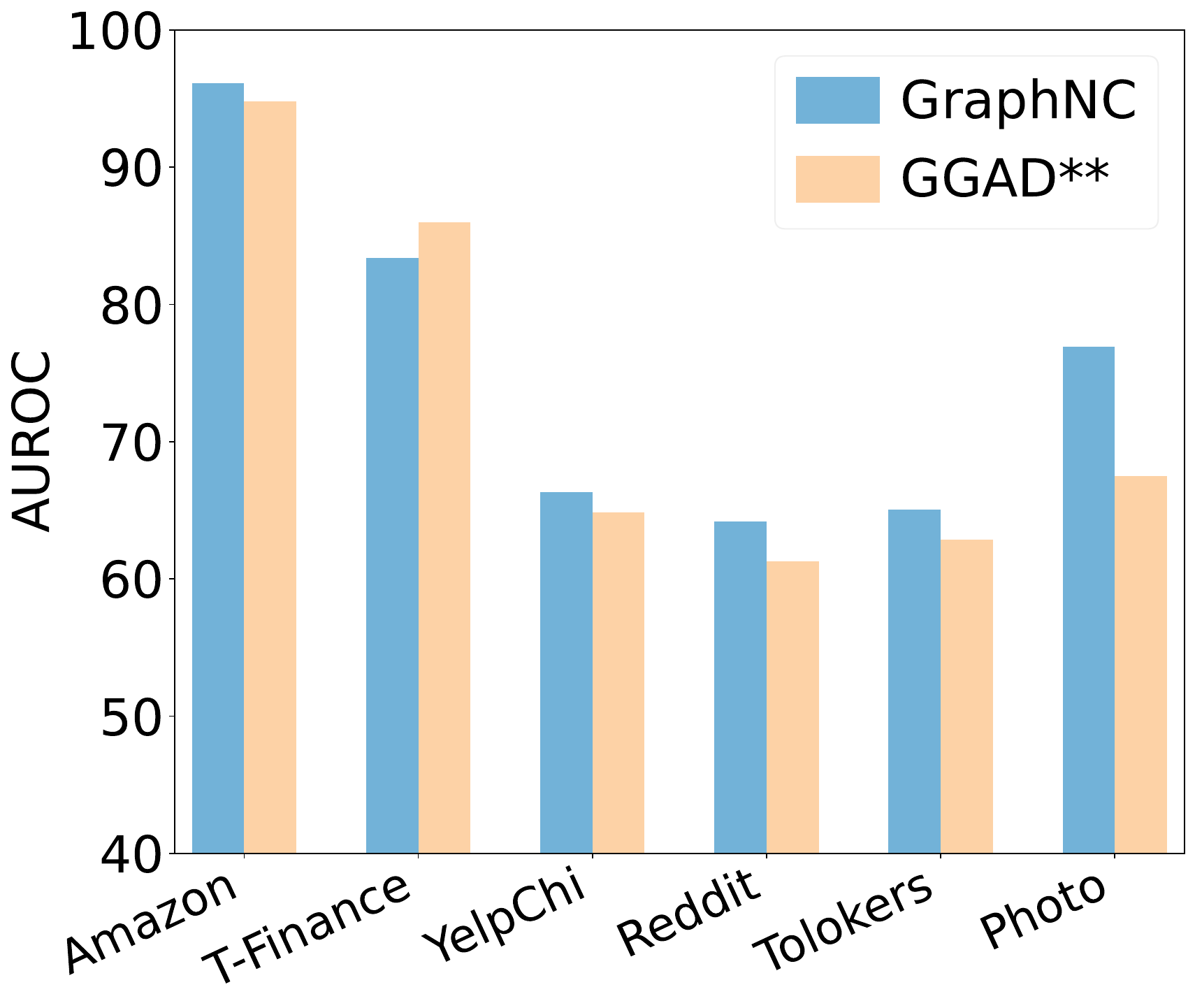}
  \caption{AUROC}
\end{subfigure}\hfill
\begin{subfigure}[t]{0.5\textwidth}
  \centering
  \includegraphics[width=0.7\linewidth]{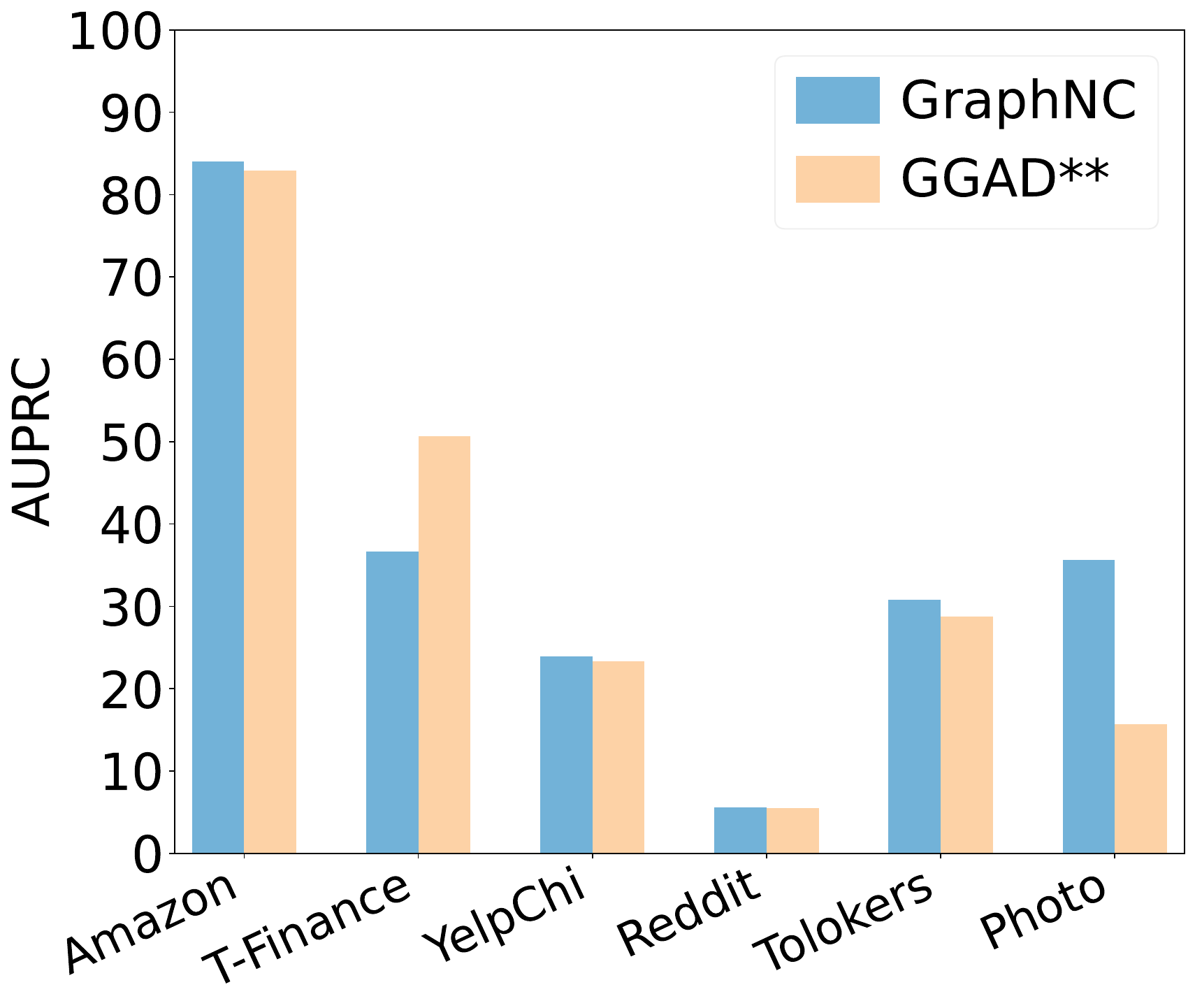}
  \caption{AUPRC}
\end{subfigure}
\caption{AUROC and AUPRC comparison with GGAD**.}
\label{fig:variant}
\end{figure}

\subsection{More Analysis on the FP and FN Results}
Unlike AUROC and AUPRC provided in the main experiments, FP and FN are threshold-dependent metrics, as they require a specific decision threshold to separate normal and abnormal nodes.
To analyze the sensitivity of the threshold $\gamma$ used to separate normal and abnormal nodes in the calculation of FP and FN, we provide a threshold analysis by varying the $\gamma$ from 0.7 to 0.9. As shown in Table \ref{tab:fpfn_thresholds}, under varying thresholds, GraphNC consistently produces fewer FP and FN than GGAD, which further highlights the effectiveness of our method.

\begin{table*}[t]
\caption{FP/FN comparison under different thresholds.}
\label{tab:fpfn_thresholds}
\begin{center}

\setlength{\tabcolsep}{14pt} 
\renewcommand{\arraystretch}{1.15} 

\begin{tabular}{c|c|cc|cc}
\hline
\hline
\multirow{2}{*}{\textbf{Threshold}} & \multirow{2}{*}{\textbf{Method}} 
& \multicolumn{2}{c|}{\textbf{Amazon}} 
& \multicolumn{2}{c}{\textbf{Tolokers}} \\
\cline{3-6}
& & \textbf{FP} & \textbf{FN} & \textbf{FP} & \textbf{FN} \\
\hline
\multirow{2}{*}{$\gamma$ = 0.7}
& GGAD    & 1721 & 51  & 1582 & 1024 \\
& GraphNC & 1710 & 40  & 1465 & 902  \\
\hline
\multirow{2}{*}{$\gamma$ = 0.8}
& GGAD    & 1029 & 66  & 1074 & 1216 \\
& GraphNC & 993  & 40  & 957  & 1099 \\
\hline
\multirow{2}{*}{$\gamma$ = 0.9}
& GGAD    & 361  & 125 & 515  & 1368 \\
& GraphNC & 315  & 79  & 470  & 1318 \\
\hline
\hline
\end{tabular}

\end{center}
\end{table*}

\subsection{Effects of Data Augmentation}
To further investigate the impact of different augmentation strategies on our results, we replaced the masking strategy in GraphNC with a random edge dropping, referred to as GraphNC-RED, and the corresponding results are presented in Fig. \ref{fig:red}. We observe that GraphNC-RED yields consistent performance with GraphNC, suggesting that perturbation strategies such as random masking and edge dropping are both effective within the GraphNC framework. The main reason is that these perturbations act as regularization mechanisms, enhancing the model’s robustness and encouraging it to capture more invariant patterns rather than relying on potentially noisy or dataset-specific signals. Since the structural properties and anomaly distributions differ across datasets, the relative benefit of each perturbation strategy may also vary. Since graph anomalies are often manifested as anomalies in node features, we adopted feature perturbation to calibrate normality during score alignment.


\begin{figure}[t]
\centering
\begin{subfigure}[t]{0.5\textwidth}
  \centering
  \includegraphics[width=0.7\linewidth]{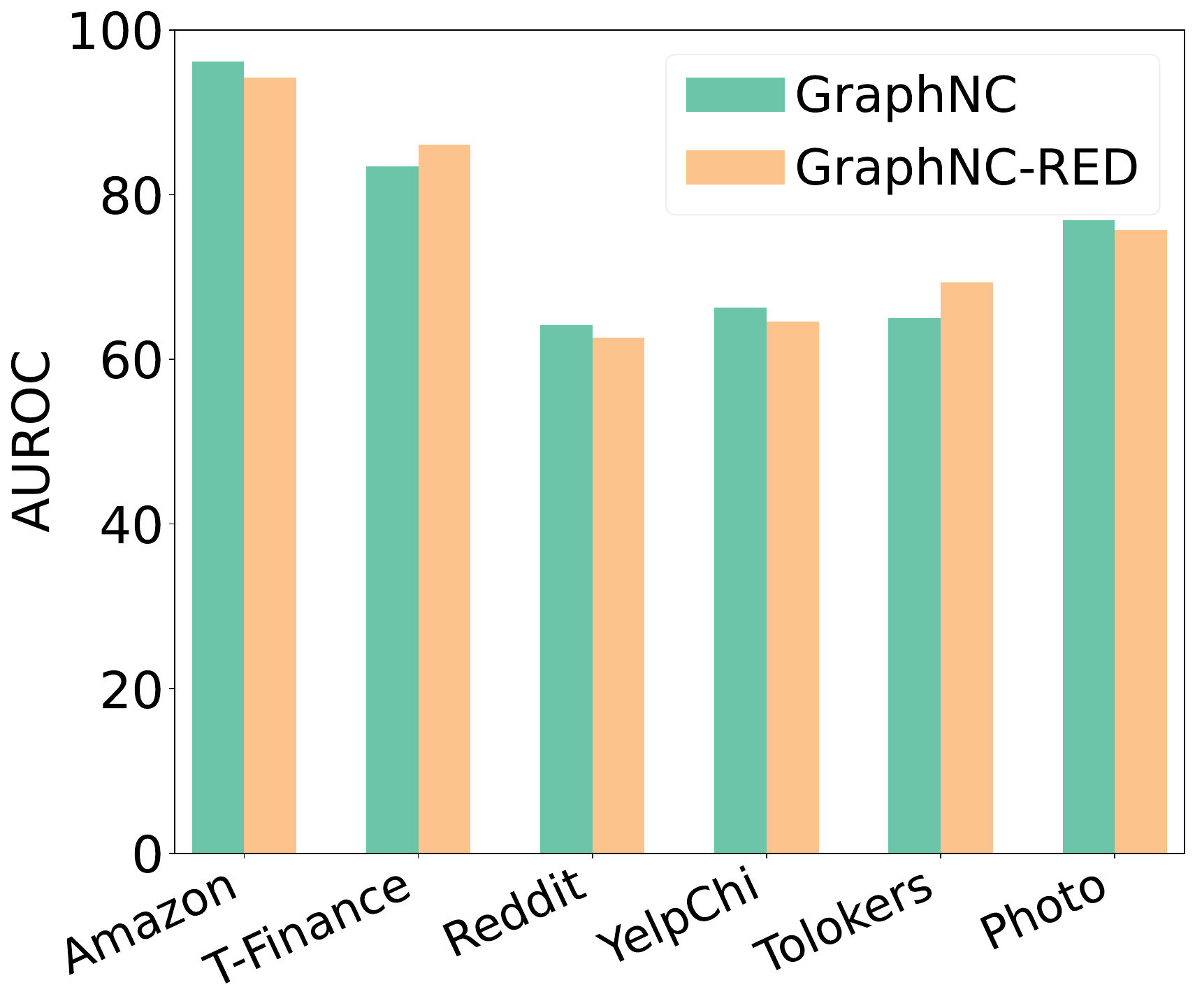}
  \caption{AUROC}
\end{subfigure}\hfill
\begin{subfigure}[t]{0.5\textwidth}
  \centering
  \includegraphics[width=0.7\linewidth]{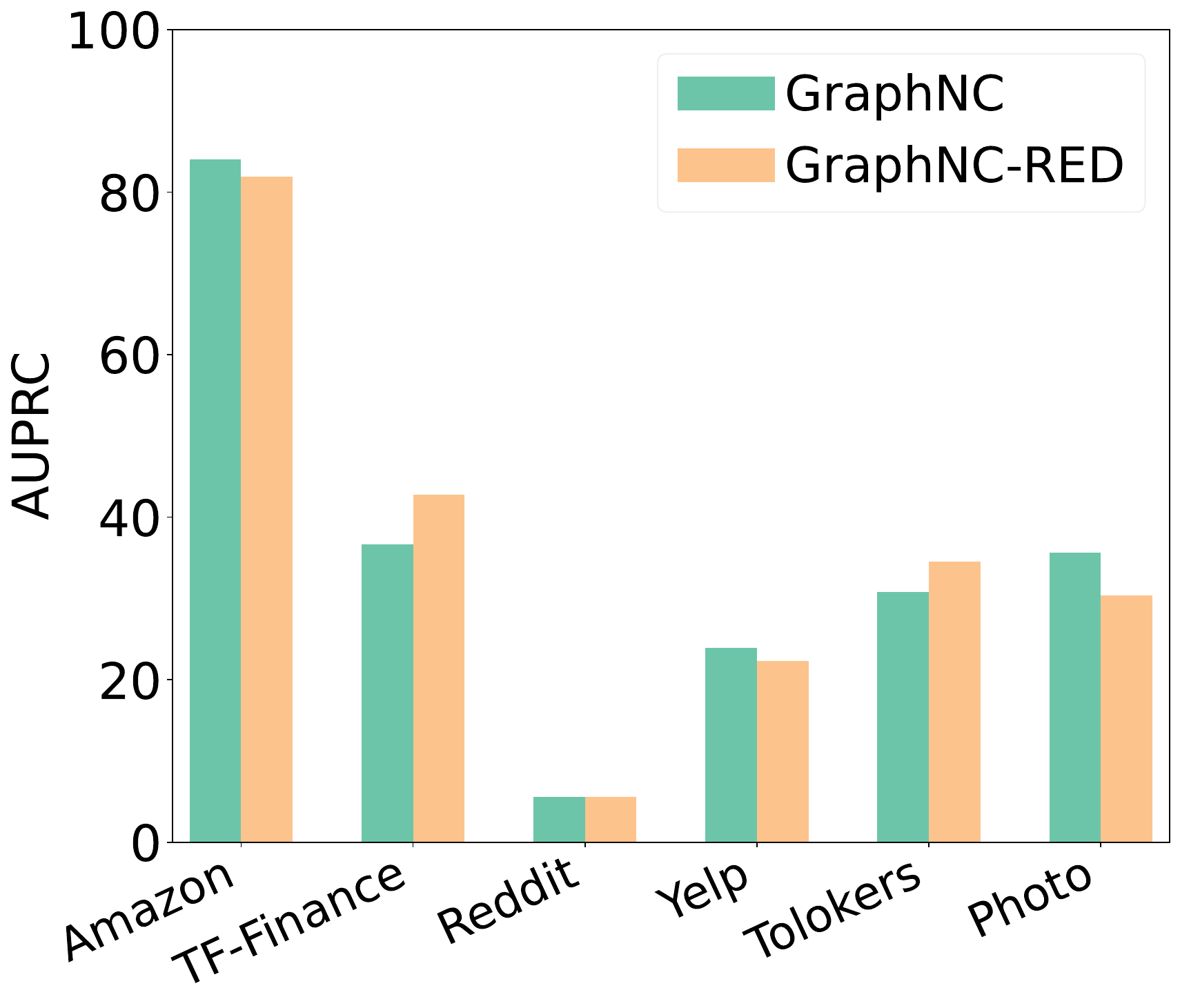}
  \caption{AUPRC}
\end{subfigure}
\caption{AUROC and AUPRC comparison with GraphNC-RED.}
\label{fig:red}
\end{figure}

\section{Theoretical Analysis} \label{theorem_analyse}
\setcounter{theorem}{0}
\begin{theorem}
Let $\sigma_\mathcal{T}^2$ and $\sigma_\mathcal{S}^2$ denote the variance of the anomaly scores yielded by the teacher model and the student model in the normal class, respectively, then under our GraphNC framework, the student model achieves shrinking score variance (\ie, $\sigma_\mathcal{S}^2 < \sigma_\mathcal{T}^2$) by minimizing $\mathcal{L}_{NormReg}$ while applying $\mathcal{L}_{ScoreDA}$, thereby reducing False Positive Rate (FPR) and False Negative Rate (FNR).
\end{theorem}

\begin{proof}
    We denote the teacher model’s score on the normal class by the random variable $T$ (\ie, $S_T|y=0$) $ \sim (\mu_0,\sigma_T^2)$, 
    which follows a sub-Gaussian distribution, with a mean of $\mu_0$ and a variance of $Var(T)=\sigma_\mathcal{T}^2$. 
    The law of total variance is given by
    \begin{equation}
        Var(T) = Var (\mathbb{E}[T|\mathbf{X}]) + \mathbb{E}(Var(T|\mathbf{X})),
    \end{equation}
    where $\mathbf{X}$ is the input feature. The  optimal solution of the student model $\mathcal{F}_{\mathcal{S}}$ learned through soft score distillation with the input feature is $\mathcal{F}_{\mathcal{S}}^{*}(\mathbf{X}) = \mathbb{E}[T|\mathbf{X}]$, then we have
    \begin{equation}
        Var(\mathcal{F}_{\mathcal{S}}^{*}(\mathbf{X})) = Var(\mathbb{E}[T|\mathbf{X}])\leq Var(T) = \sigma_\mathcal{T}^2.
    \end{equation}
  By optimizing the $\mathcal{L}_{ScoreDA}$, the variance of the predicted scores of the student model on normal nodes is usually not larger than that of the teacher model.
    
    Furthermore, the designed $\mathcal{L}_{NormReg}$ loss further shrinks the intra-class variance of normal classes on the student model. $\mu_0$ can be regarded as the potential true center of the normal class, and then the representation of each normal node $v_i$ in the two views of the student model can be written as $h_i = \mu_0 + \epsilon_i, \quad \widetilde{h}_i = \mu_0 + \widetilde{\epsilon}_i$, where $\epsilon_i$ and $\widetilde{\epsilon}_i$ are the perturbations or noise introduced by two different views. We assume that the perturbations in both views are statistically independent and have zero mean, \ie, $\mathbb{E}[\epsilon_i] = \mathbb{E}[\widetilde{\epsilon}_i] = 0$. By substituting $h_i$ and $\widetilde{h}_i$ into the $\mathcal{L}_{NormReg}$, we can redefine it as
    \begin{equation}
        \mathcal{L}_{NormReg} = \frac{1}{|{\cal V}_l|} \sum_{v_i \in {\cal V}_l} \|\epsilon_i - \widetilde{\epsilon}_i\|_2^2 
    \end{equation}    
    By unfolding the loss and then computing its expectations, we obtain 
    \begin{equation}
    \begin{aligned}
        \mathbb{E}[\|\epsilon_i - \widetilde{\epsilon}_i\|_2^2] & =  \mathbb{E}[\|\epsilon_i\|_2^2 + \|\widetilde{\epsilon}_i\|_2^2-2 \langle \epsilon_i, \widetilde{\epsilon}_i \rangle]\\
       & = \mathbb{E} [\|\epsilon_i\|_2^2] + \mathbb{E}[\|\widetilde{\epsilon}_i\|_2^2] - 2 \langle \mathbb{E} [\epsilon_i], \mathbb{E} [\widetilde{\epsilon}_i] \rangle
        \end{aligned}
    \end{equation}
    Since $\epsilon_i$ and $\widetilde{\epsilon}_i$ are independent, and  $\langle
    \mathbb{E} [\epsilon_i], \mathbb{E} [\widetilde{\epsilon}_i] \rangle =0$. We have,
    \begin{equation}
        \mathbb{E}[\|\epsilon_i - \widetilde{\epsilon}_i\|_2^2] = \mathbb{E} [\|\epsilon_i\|_2^2] + \mathbb{E}[\|\widetilde{\epsilon}_i\|_2^2] = \sigma_{\epsilon}^2 + \sigma_{\widetilde{\epsilon}}^2,
    \end{equation}
    where $\sigma_{\epsilon}^2$ and $\sigma_{\widetilde{\epsilon}}^2$ are the population variances of the two views in the student model, respectively.  By minimizing $\mathcal{L}_{NormReg}$, the node embeddings tend to be consistent across the two views (\ie, the consistent variances pattern) and the population variance decreases, making the normal node embeddings more compactly clustered around $\mu_0$. It is known that $Var(\mathcal{F}_{\mathcal{S}}^{*}(\mathbf{X}))$ corresponds to the variance of the normal class in the student model when considering only the distillation loss $\mathcal{L}_{ScoreDA}$. Therefore, when $\mathcal{L}_{NormReg}$ is taken into account, the variance of the normal class in the student model satisfies $\sigma_{\mathcal{S}}^2 = Var(\mathcal{F}_{\mathcal{S}}^{*}(\mathcal{H})) < Var(\mathcal{F}_{\mathcal{S}}^{*}(\mathbf{X}))\leq \sigma_{\mathcal{T}}^2$, where $\mathcal{H}$ is the output of the student model.
    
    Since $S_T|y=0$ follows a sub-Gaussian distribution, there exists $\sigma_{\mathcal{T}}$ such that for $t>0$ it satisfies
    \begin{equation}\label{the_Gau}
        \mathbb{P}(S_T-\mu_0 \geq t) \leq exp(-\frac{t^2}{2\sigma_{\mathcal{T}}^2}) 
    \end{equation}
    Given an arbitrary fixed threshold $\tau$ ($\mu_0<\tau$), if the predicted score of a sample exceeds this threshold, the sample is classified as abnormal. According to Eq. (\ref{the_Gau}) and taking $t = \tau-\mu_0$, we have,
    \begin{equation}
         \mathbb{P}(S_s \geq \tau|y=0) \leq exp(-\frac{(\tau-\mu_0)^2}{2\sigma_{\mathcal{S}}^2}), \quad \mathbb{P}(S_T \geq \tau|y=0) \leq exp(-\frac{(\tau-\mu_0)^2}{2\sigma_{\mathcal{T}}^2})
    \end{equation}
    Since $\sigma_{\mathcal{S}}^2 < \sigma_{\mathcal{T}}^2$, $\mathbb{P}(S_s \geq \tau|y=0) < \mathbb{P}(S_T \geq \tau|y=0)$. It indicates that the probability of the student model wrongly identifying an abnormal node as a normal node is lower than that of the teacher model, thereby reducing FNR.

   Similarly, if the predicted score is below this threshold, the node is classified as normal. We compute the probabilities of the student model ($\mathbb{P}(S_s< \tau|y=0)$) and the teacher model ($\mathbb{P}(S_T < \tau|y=0)$) predicting it as a normal node as follows,
    \begin{equation}
    \mathbb{P}(S_s< \tau|y=0) = 1- \mathbb{P}(S_s \geq \tau|y=0) , \quad \mathbb{P}(S_T< \tau|y=0) = 1- \mathbb{P}(S_T \geq \tau|y=0).
    \end{equation}
    Since $\mathbb{P}(S_s \geq \tau|y=0) < \mathbb{P}(S_T \geq \tau|y=0)$, we have,
    \begin{equation}
        1- \mathbb{P}(S_s \geq \tau|y=0) \geq 1- \mathbb{P}(S_T \geq \tau|y=0).
    \end{equation}
    Finally, we derive $\mathbb{P}(S_s < \tau|y=0) > \mathbb{P}(S_T < \tau|y=0)$, which can equivalently be expressed as $\mathbb{P}(S_s < \tau|y=1) < \mathbb{P}(S_T < \tau|y=1)$, given that the anomaly detection task uses binary labels ($y=0$ for normal and $y=1$ for abnormal). It indicates that the probability of misclassifying a normal node as abnormal is lower in the student model than in the teacher model, leading to a reduced FPR.
\end{proof}

\section{Efficiency Analysis} \label{app:time_complexity}
\subsection{Time Complexity Analysis} \label{app:time_complexity}
In this section, we analyze the time complexity of GraphNC. Since it is designed as a plug-in framework that can be integrated with various teacher models to enhance GAD, we focus here only on the complexity of the additional module itself. We employ the combination of GCN and one MLP layer as the student model in GraphNC. The GCN takes $\mathcal{O}(EM+NMd)$, where $E$ is the number of edges in the graph, $N$ is the number of nodes, $M$ is the dimension of the attribute, and $d$ is the dimension of representation. The MLP used for feature transformation in GraphNC is $\mathcal{O}(NMd)$. In GraphNC, ScoreDA aims to align the student score with the output score distribution, which will take $O(N)$.  The NormReg aiming to  minimize the distance between two views, which will also take $O(N)$ 
The total complexity is $\mathcal{O}(EM+2NMd+2N) + \bf{T}$, where $\bf{T}$ is the time complexity of the corresponding teacher model.

\subsection{Runtime Analysis}
We provide a runtime comparison between the standalone teacher model and the model augmented with GraphNC below. The training and inference time comparison are shown in Fig. \ref{fig:runtime}.  It is clear that GraphNC is highly efficient for training, as its training process builds directly upon the pretrained teacher model. In addition, GGAD and GraphNC are both very efficient for inference.

\begin{figure}[t]
\centering
\begin{subfigure}[t]{0.5\textwidth}
  \centering
  \includegraphics[width=0.7\linewidth]{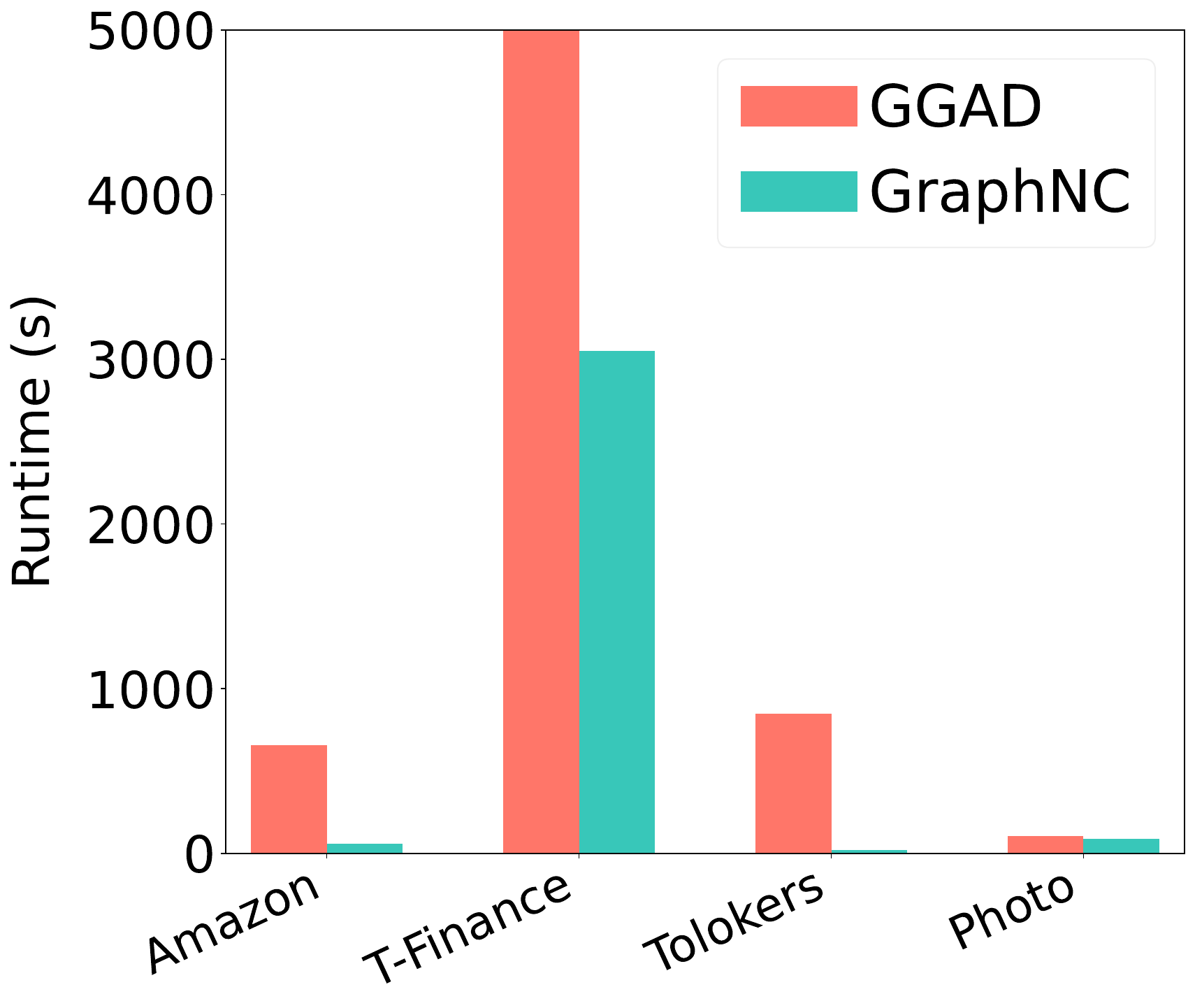}
  \caption{Running times comparison of training}
\end{subfigure}\hfill
\begin{subfigure}[t]{0.5\textwidth}
  \centering
  \includegraphics[width=0.7\linewidth]{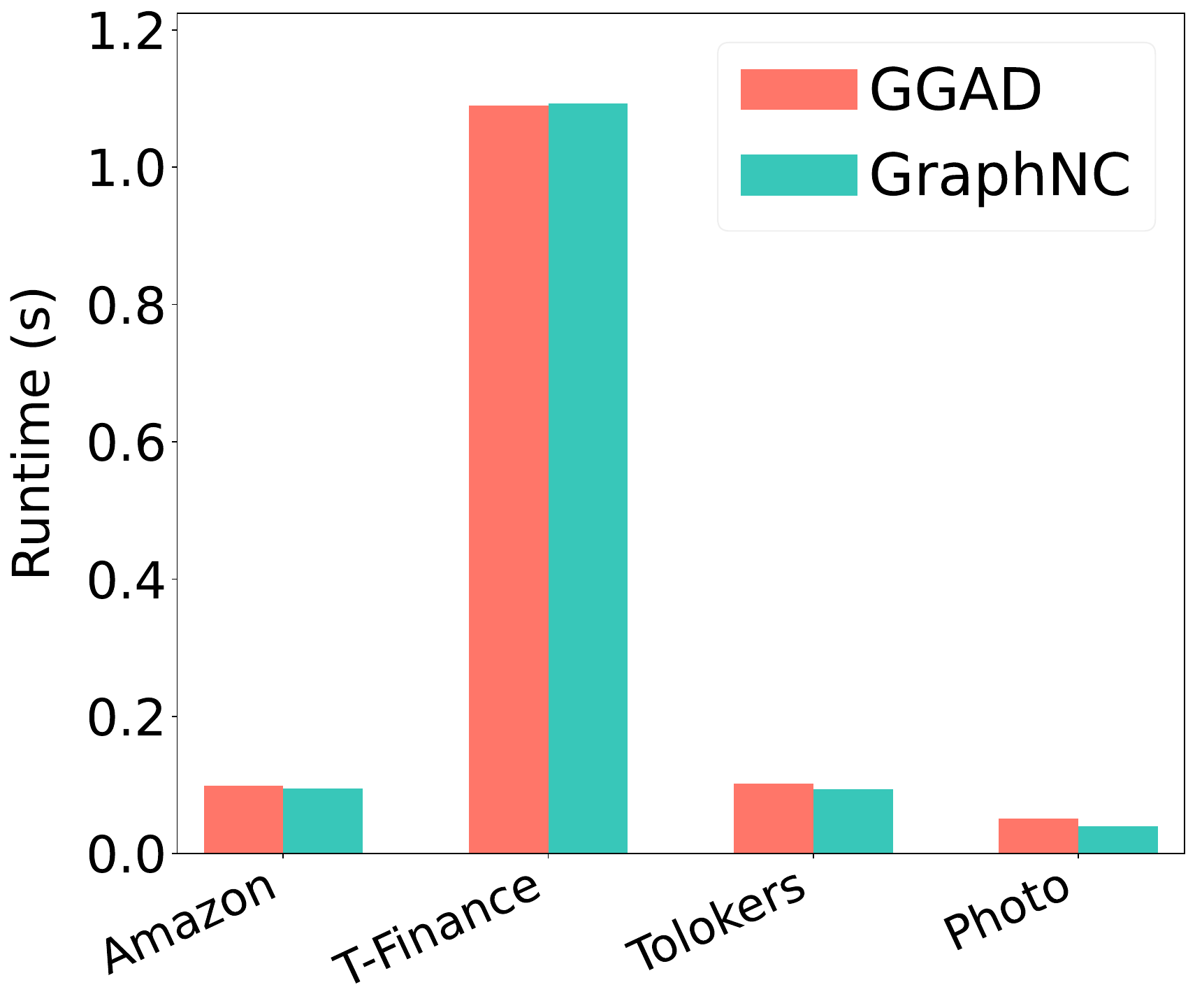}
  \caption{Running times comparison of inference}
\end{subfigure}
\caption{Training and inference runtimes (in seconds) on the six datasets for GGAD and GraphNC.}
\label{fig:runtime}
\end{figure}

\section{Algorithm}
The algorithm of GraphNC is summarized in Algorithm~\ref{alg:student_training_final}.

\begin{algorithm}[tb]
\caption{GraphNC}
\label{alg:student_training_final}
\begin{algorithmic}[1]
\STATE \textbf{Input:} Graph $\mathcal{G}=(\mathcal{V},\mathcal{E},\mathbf{X})$; pre-trained semi-supervised GAD model $\mathcal{F}_{\mathcal{T}}$; student model $\mathcal{F}_{S}$ with parameters $\Phi$; labeled normal nodes $\mathcal{V}_l \subset \mathcal{V}$; training epochs $E$; weight $\alpha$.
\STATE \textbf{Output:} Anomaly scores $\{y_i^{\mathcal{S}}\}_{v_i\in\mathcal{V}}$.

\STATE Initialize parameters $\Phi$ for the student model $\mathcal{F}_{S}$.
\STATE Obtain teacher score distribution $\mathcal{Y}^{\mathcal{T}}=\mathcal{F}_{\mathcal{T}}(\mathbf{X},\mathcal{G})=\{y_1^{\mathcal{T}},\dots,y_N^{\mathcal{T}}\}$.
\STATE Create augmented features $\tilde{\mathbf{X}}$ by applying $\operatorname{RandomMask}$ to features of nodes in $\mathcal{V}_l$.
\STATE \textit{// GraphNC Training}
\FOR{$e=1$ to $E$}
    \STATE Compute student representations $\mathbf{H}^{\mathcal{S}} \gets \mathcal{F}_{GNN}(\mathbf{X},\mathcal{G};\Omega)$.
    \STATE Compute augmented representations $\tilde{\mathbf{H}}^{\mathcal{S}} \gets \mathcal{F}_{GNN}(\tilde{\mathbf{X}},\mathcal{G};\Omega)$.
    \STATE Compute student scores $\mathcal{Y}^{\mathcal{S}} \gets \mathcal{F}_{MLP}(\mathbf{H}^{\mathcal{S}};\phi)$.

    \STATE \textit{// Anomaly Score Distribution Alignment (ScoreDA)}
    \STATE $\mathcal{L}_{ScoreDA} \gets \frac{1}{|\mathcal{V}|}\sum_{v_i\in\mathcal{V}} \|y_i^{\mathcal{S}}-y_i^{\mathcal{T}}\|_2^2$.

    \STATE \textit{// Perturbation-based Normality Regularization (NormReg)}
    \STATE $\mathcal{L}_{NormReg} \gets \frac{1}{|\mathcal{V}_l|}\sum_{v_i\in\mathcal{V}_l} \|\mathbf{h}_i^{\mathcal{S}}-\tilde{\mathbf{h}}_i^{\mathcal{S}}\|_2^2$.

    \STATE $\mathcal{L}_{GraphNC} \gets \mathcal{L}_{ScoreDA} + \alpha \cdot \mathcal{L}_{NormReg}$.
    \STATE Update $\Phi=\{\Omega,\phi\}$ by minimizing $\mathcal{L}_{GraphNC}$.
\ENDFOR
\STATE \textit{// GraphNC Anomaly Score Inference}
\FOR{$i=1$ to $N$}
    \STATE Load the test graph $\mathcal{G}$ and optimal parameters $\Phi^*$ for the student model $\mathcal{F}_{S}$.
    \STATE  Anomaly Scoring $y_i^{\mathcal{S}} = \mathcal{F}_{ \mathcal{S}
    }(v_i, {\bf{x}}_i, \mathcal{G}, \Phi^*)$
\ENDFOR
\STATE \textbf{return} $\{y_i^{\mathcal{S}}\}_{v_i\in\mathcal{V}}$.

\end{algorithmic}
\end{algorithm}


\end{document}